\documentclass[nohyperref]{article}

\usepackage{microtype,etoc}
\usepackage{graphicx}
\usepackage{subfigure}
\usepackage{booktabs}
\usepackage{hyperref}
\usepackage[accepted]{icml2023}
\usepackage{amsmath}
\usepackage{amssymb}
\usepackage{bm}
\usepackage{mathtools}
\usepackage{amsthm,threeparttable}
\usepackage[capitalize,noabbrev]{cleveref}
\usepackage{multicol}
\usepackage{multirow}
\usepackage{pifont}
\usepackage{amsbsy}
\usepackage[normalem]{ulem}
\usepackage{enumitem}
\usepackage{thm-restate}

\definecolor{fhcolor}{rgb}{0.523, 0.235, 0.625}

\theoremstyle{plain}
\newtheorem{theorem}{Theorem}
\newtheorem{Proposition}{Proposition}
\newtheorem{lemma}{Lemma}

\theoremstyle{definition}
\newtheorem{definition}{Definition}
\newtheorem{assumption}{Assumption}
\theoremstyle{remark}

\crefname{section}{Section}{Section}
\crefname{equation}{Eq.}{Eq.}

\providecommand{\realnum}{\mathbb{R}}

\definecolor{mygreen}{RGB}{30, 180, 50}

\usepackage[textsize=tiny]{todonotes}

\usepackage{selectp} 

\icmltitlerunning{Benign Overfitting in Deep Neural Networks under Lazy Training}

\begin{document}

\etocdepthtag.toc{mtchapter}
\etocsettagdepth{mtchapter}{subsection}
\etocsettagdepth{mtappendix}{none}

\twocolumn[
\icmltitle{Benign Overfitting in Deep Neural Networks under Lazy Training}

\begin{icmlauthorlist}
\icmlauthor{Zhenyu Zhu}{EPFL}
\icmlauthor{Fanghui Liu}{EPFL}
\icmlauthor{Grigorios G Chrysos}{EPFL}
\icmlauthor{Francesco Locatello}{AWS}
\icmlauthor{Volkan Cevher}{EPFL}
\end{icmlauthorlist}

\icmlaffiliation{EPFL}{Laboratory for Information and Inference Systems, \'{E}cole Polytechnique F\'{e}d\'{e}rale de Lausanne (EPFL), Switzerland}
\icmlaffiliation{AWS}{Amazon Web Services (Work done outside of Amazon)}

\icmlcorrespondingauthor{Zhenyu Zhu}{zhenyu.zhu@epfl.ch}
\icmlcorrespondingauthor{Fanghui Liu}{fanghui.liu@epfl.ch}

\icmlkeywords{Machine Learning, ICML}

\vskip 0.3in
]

\printAffiliationsAndNotice{}

\begin{abstract}
This paper focuses on over-parameterized deep neural networks (DNNs) with ReLU activation functions and proves that when the data distribution is well-separated, DNNs can achieve \textit{Bayes-optimal} test error for classification while obtaining (nearly) zero-training error under the lazy training regime. For this purpose, we unify three interrelated concepts of overparameterization, benign overfitting, and the Lipschitz constant of DNNs. Our results indicate that interpolating with smoother functions leads to better generalization. Furthermore, we investigate the special case where interpolating smooth ground-truth functions is performed by DNNs under the Neural Tangent Kernel (NTK) regime for generalization. Our result demonstrates that the generalization error converges to a constant order that only depends on label noise and initialization noise, which theoretically verifies benign overfitting. Our analysis provides a tight lower bound on the normalized margin under non-smooth activation functions, as well as the minimum eigenvalue of NTK under high-dimensional settings, which has its own interest in learning theory. 
\end{abstract}

\section{Introduction}
\label{sec:introduction}
\looseness=-1Benign overfitting has attracted significant research interest recently in an effort to understand why predictors with zero training loss can still achieve counter-intuitively good generalization performance even in the presence of noise~\citep{koehler2021uniform, zou2021benign,chatterji2022foolish,wang2022tight,mei2022generalization}. 
Current efforts on benign overfitting mainly focus on the finite sample behavior under linear regression~\citep{Bartlett_2020, chatterji2021interplay,zou2021benign}, kernel-based estimators~\citep{mei2022generalization,Liang2019MultipleDescent}, and logistic regression~\citep{Montanari2019generalization,wang2021benign}.

To our knowledge, results on neural networks (NNs) are restricted to two-layer neural networks \cite{tsigler2020benign, ju2021generalization, frei2022benign, cao2022benign} and three-layer neural networks but only the last layer is trained \cite{ju2022on}. The extension from shallow NNs to deep neural networks (DNNs) is non-trivial:
\emph{under what conditions does benign overfitting occur in deep neural networks?} and \emph{what makes them special?} are still open problems in both statistical learning theory and deep learning theory.

In this work, we address this open question in benign overfitting of deep ReLU NNs for binary classification under the lazy training regime.
We assume the network is trained by stochastic gradient descent (SGD) on well-separated data under adversarially corrupted labels, following the standard problem setting of \citet{frei2022benign}.
We prove that the  ReLU DNN exhibits benign overfitting, i.e., obtaining \textit{Bayes-optimal} test error while obtaining zero training error under the lazy training regime.

Our results establish a rigorous connection between the Lipschitz constants of DNNs and benign overfitting. We demonstrate that interpolating with (Lipschitz) smoother functions leads to a faster convergence rate on the generalization guarantees. 
Accordingly, for a better understanding of how the estimator by DNNs interpolates the ground-truth function, we also consider a regression task for DNNs from an approximation theory view \cite{cucker2007learning}, interpolating the smooth ground-truth function by DNNs under the neural tangent kernel (NTK) regime~\citep{Bach2017,jacot2018neural}.

Overall, we expect our results to foster a refined analysis of the generalization guarantees for large dimensional machine learning models, especially on DNNs.

\subsection{Contributions and technical challenges}
In this paper, we consider a finite sample behavior, in which the input dimension $d$ can be large but fixed, or comparably large with the number of training data $n$ and model parameters to obtain a dimension-free bound~\citep{Bartlett_2020,ju2022on,li2021towards}. Our main contributions are summarized below:

\begin{itemize}
    \item We adhere to the standard data setting with label noise that has been previously explored in two-layer networks~\citep{frei2022benign}, along with the model setting for multi-layer fully connected neural networks~\citep{cao2019generalization, allen2019convergence}. Building upon the proof concept of~\citep{frei2022benign}, we extend our results to deep ReLU neural networks, which presents a considerable challenge in connecting the Bayes-optimal test error to the training dynamics of deep neural networks. \cref{thm:main_thm} for binary classification shows that, under the lazy training regime, even though training on noisy data, DNNs can still obtain the \textit{Bayes-optimal} test error, i.e., the error rate is less than the proportion of incorrect labels in the training set plus a generalization term that converges to zero. We also demonstrate that this term is positively correlated with the Lipschitz constants of DNNs, which implies that interpolating more smooth functions leads to a faster convergence rate.
    \item \cref{thm:min_eigen_NTK} provides the first lower bound on the minimum eigenvalues of the NTK matrix of DNNs in the high-dimensional setting and demonstrates its phase transition under different tendencies of the number of training data and input dimension. We believe that it has its own interest in learning theory.
    \item \cref{thm:kernel} builds the generalization guarantees of over-parameterized neural networks under the NTK regime in the high-dimensional setting to learn a ground-truth function in RKHS. 
    Our result exhibits a phase transition on the excess risk (related to generalization performance) between the $n<d$ and $n>d$ case.
    It implies that the excess risk finally converges to a constant order only relying on the label noise and initialization noise, which theoretically verifies the benign overfitting.
\end{itemize}

{\bf Technical challenges.} 
The main technical challenge of this paper is how to derive the lower bound of the non-smooth function in deep ReLU neural networks for the normalized margin on test points. In the context of lazy training~\citep{NEURIPS2019_ae614c55}, the function of a neural network is nearly linear during the initial stages of training. By analyzing the accumulation of weights for each training step, we can establish a lower bound for the normalized margin on test points. By doing so, we transform the \textit{Bayes-optimal} test error to the expected risk and Lipschitz constant of DNNs.

When compared with~\citet{frei2022benign} on shallow neural networks with smooth activation functions for binary classification, we extend their results to our deep ReLU neural networks, not limited to high-dimensional settings, and obtain a faster convergence rate with the number of data. 
The key difficulty lies in how to build the relationship between the \textit{Bayes-optimal} test error and the training dynamics of DNNs. When compared to the generalization results of deep neural networks (DNNs) in the over-parameterized regime~\citep{cao2019generalization}, our results focus on overfitted models that are trained by noisy data and achieve a faster convergence rate. Besides, \citet{ju2022on} present generalization guarantees on three-layer neural networks (only the last layer is training) for regression, which has the closed-form min $L^2$-norm solution.
However, this nice property is invalid in our DNN setting.
In this case, we build the connection between DNNs and kernel methods (e.g., NTK) in high dimensional settings for benign overfitting.

\subsection{Related work}
{\bf Benign overfitting:} There has been a significant amount of research devoted to understanding the phenomenon of benign overfitting, with a particular emphasis on linear models, e.g., linear regression~\citep{Bartlett_2020, chatterji2021interplay,zou2021benign}, sparse linear regression~\citep{chatterji2022foolish, koehler2021uniform,wang2022tight}, logistic regression~\citep{Montanari2019generalization,wang2021benign}, ridge regression~\citep{tsigler2020benign} and kernel-based estimators~\citep{mei2022generalization,Liang2019MultipleDescent}.
Furthermore, the concept of benign overfitting can be extended to tempered or catastrophic based on various spectra of the kernel (ridge) regression \cite{mallinar2022benign}. 

For nonlinear models,~\citet{li2021towards} study the benign overfitting phenomenon of random feature models. \citet{frei2022benign} prove that a two-layer fully connected neural network exhibits benign overfitting under certain conditions, e.g., well-separated log-concave distribution and smooth activation function. Then~\citet{Xingyu2023Benign} extends the previous results to the non-smooth case. Similarly, \citet{cao2022benign} focus on the benign overfitting of two-layer convolutional neural networks (CNN). 

\citet{mallinar2022benign} argue that many true interpolation methods (such as neural networks) for noisy data are not benign but tempered overfitting, and even catastrophic under various model capacities.

{\bf Generalization of NNs and Neural Tangent Kernel (NTK):} The generalization ability of neural networks has been a core problem in machine learning theory. \citet{brutzkus2017sgd} show that SGD can learn an over-parameterized two-layer neural network with good generalization ability. \citet{allen2019learning} study the generalization performance of SGD for $2$- and $3$-layer networks.~\citet{cao2019generalization} study the training and generalization of deep neural networks (DNNs) in the over-parameterized regime. Besides,~\citet{arora2019fine,cao2020generalization,E_2020} provide the algorithm-dependent generalization bounds for different settings. 

The NTK~\citep{jacot2018neural} is a powerful tool for deep neural network analysis. Specifically, NTKs establish equivalence between the training dynamics of gradient-based algorithms for DNNs and kernel regression under specific initialization, so it can be considered as an intermediate step between simple linear models and DNNs~\citep{allen2019convergence, du2019gradient, chen2020much}. Besides, the convergence rate~\citep{arora2019exact} and generalization bound~\citep{cao2019generalization,zhu2022generalization,nguyen2021tight,bombari2022memorization} can be linked to the minimum eigenvalue of the NTK matrix. 

One can see that studying benign overfitting for DNNs is missing and it appears possible to borrow some ideas from deep learning theory, e.g., NTK. Nevertheless, we need to tackle the noisy training as well as the high-dimensional setting for benign overfitting, which is our main interest. 

\section{Problem Settings}
\label{sec:preliminaries}
In this section, we detail the problem setting for a deep ReLU neural network trained by SGD from the perspective of notations, neural network architecture, initialization schemes, and optimization algorithms. 

\subsection{Notation}
In this paper, we use the shorthand $ [n]:= \{1,2,\dots, n \}$ for a positive integer $n$.
We denote by $a(n) \gtrsim b(n)$: the inequality $a(n) \geq c b(n)$ that hides a positive constant $c$ that is independent of $n$. Vectors (matrices) are denoted by boldface, lower-case (upper-case) letters.
The standard Gaussian distribution is $\mathcal{N}(0, 1)$ with the zero-mean and the identity variance. 
We use the $\text{Lip}_{f}$ to represent the Lipschitz constant of the function $f$.
We follow the standard Bachmann–Landau notation in complexity theory e.g., $\mathcal{O}$, $o$, $\Omega$, and $\Theta$ for order notation.

\subsection{Network}
\label{ssec:network}

Here we introduce the formulation of DNNs.
We focus on the typical depth-$L$ fully-connected ReLU neural networks with scalar output, width $m$ on the hidden layers and $n$ training data, $\forall i\in[n]$:

\begin{equation}
\begin{matrix}
\bm{h}_{i,0} = \bm{x}_i; \\
\\
\bm{h}_{i,l} = \phi(\bm{W}_l\bm{h}_{i,l-1}); \quad \forall l\in[L-1]; \\
\\
f(\bm x_i; \bm W) = \bm{W}_L\bm{h}_{i,L-1};\\
\\
\end{matrix}
\label{eq:deep_network}
\end{equation}

where $\bm{x}_i\in \mathbb{R}^d$ is the input, $f(\bm{x}_i; \bm W) \in \mathbb{R}$ is the neural network output, and $\phi = \max(0,x)$ is the ReLU activation function. The neural network parameters formulate the tuple of weight matrices $\bm{W} := \{ \bm{W}_l \}_{l=1}^L \in  \{ \mathbb{R}^{m\times d} \times (\mathbb{R}^{m\times m})^{L-2}\times \mathbb{R}^{1 \times m} \}$.

{\bf Initialization:} 
We follow the standard Neural Tangent Kernel (NTK) initialization~\citep{allen2019convergence}:
\begin{equation}
\begin{matrix}
[\bm{W}_1]_{i,j}\sim \mathcal{N}(0,\frac{2}{m}); \quad \forall i, j \in [m] \times [d]; \\
\\
[\bm{W}_l]_{i,j}\sim \mathcal{N}(0,\frac{2}{m}); \quad \forall i, j \in [m] \quad\text{and}\quad l \in [L-2]+1; \\
\\
[\bm{W}_L]_{i,j}\sim \mathcal{N}(0,1); \quad \forall i, j \in [1] \times [m]. \\
\end{matrix}
\label{eq:initialization}
\end{equation}

The related Neural Tangent Kernel (NTK)~\citep{jacot2018neural} matrix of neural network $f$ can be expressed as:
\begin{equation}
K_{\tt NTK}(\bm{x},\widetilde{\bm{x}}) := \mathbb{E}_{\bm{W}}\left \langle \frac{\partial f(\bm{x};\bm{W})}{\partial\bm{W}},\frac{\partial f(\widetilde{\bm{x}};\bm{W})}{\partial\bm{W}}  \right \rangle\,.
\label{eq:NTK}
\end{equation}

By virtue of $\phi(x) = x\phi^{\prime}(x)$ of ReLU, we have $\bm{h}_{i,l} = \bm{D}_{i,l}\bm{W}_l\bm{h}_{i,l-1}$, where $\bm{D}_{i,l}$ is a diagonal matrix under the ReLU activation function defined as below.

\begin{definition}[Diagonal sign matrix]
\label{def:diagonal_sign_matrix}
For each $i \in [n]$, $l\in [L-1]$ and $k \in [m]$, the diagonal sign matrix $\bm{D}_{i,l}$ is defined as: $(\bm{D}_{i,l})_{k,k} = 1\left \{ (\bm{W}_l\bm{h}_{i,l-1})_k \geq 0 \right \} $.
\end{definition}

In addition, we define $\omega $-neighborhood to describe the difference between two matrices.

For any $\bm{W} \in \mathcal{W} $, we define its $\omega $-neighborhood as follows:

\begin{definition}[$\omega $-neighborhood]
\label{def:omega_neighborhood}
\begin{equation*}
\mathcal{B} (\bm{W},\omega ) :=\left \{ \bm{W}'\in \mathcal{W}: \left \| \bm{W}'_l - \bm{W}_l \right \|_{\mathrm{F}} \leq \omega , l \in [L]   \right \}\,.
\end{equation*}
\end{definition}

\subsection{Optimization algorithm}
In our work, a deep ReLU neural network is trained by SGD on the training data $\{ (\bm x_i, y_i) \}_{i=1}^n$ sampled from a joint distribution $P$.
The data generation process is deferred to \cref{ssec:data} for binary classification and \cref{sec:kernel} for regression.
We employ the logistic loss for classification, which is defined as $\ell (z) = \log(1+\exp(-z))$, and denote 
$g(z) := -\ell'(z) = \frac{1}{1+e^{z}}$ for notational simplicity.

The expected risk is defined as $ \mathbb{E}_{(\bm {x},y)\sim P}~ \ell \left(y f(\bm {x};\bm W)\right)$.
Denote the empirical risks under $\ell$ by: $\hat{L}(\bm{W}):= \frac{1}{n} \sum_{i=1}^{n}\ell \left(y_i f(\bm {x}_i;\bm W)\right)$, we employ SGD to minimize $\hat{L}(\bm{W})$ initialized at $\bm{W}^{(0)}$ with fixed step-size $\alpha > 0$, as shown in~\cref{alg:algorithm_SGD}. 

For notational simplicity, at step $t$, the neural network output is denoted as $f_i^{(t)} = f(\bm {x}_i;\bm{W}^{(t)})$ and the derivative of the loss function is realted to $g_i^{(t)} := g(y_i f_i^{(t)}) = g(y_i f(\bm {x}_i;\bm{W}^{(t)}))$.

\begin{algorithm}[t]
\caption{SGD for training DNNs}
\label{alg:algorithm_SGD}
\begin{algorithmic}
\STATE {\bfseries Input:} training data $\{ (\bm x_i, y_i) \sim P\}_{i=1}^n $ and step size $\alpha$.\\
\STATE   Gaussian initialization: $\bm{W}_l^{(0)} \sim \mathcal{N}(0,2/m)$, $l \in [L-1]$.\\
\STATE   Gaussian initialization: $\bm{W}_L^{(0)} \sim \mathcal{N}(0,1)$.\\
\FOR{$i=1$ {\bfseries to} $n$}
\STATE{
Draw $(\bm {x}_i, y_i  )$ from $\{ (\bm x_i, y_i)_{i=1}^n \}$.\\
$\bm{W}^{(i)} = \bm{W}^{(i-1)}- \alpha\cdot \nabla_{\bm{W}} \ell \big( y_i f(\bm x_i; \bm{W}^{(i-1)})  \big)\,.$} 
\ENDFOR \\
\textbf{Output}  $\bm{W}^{(n)}$ for the final network $f(\bm {x};\bm{W}^{(n)})$.
\end{algorithmic}
\end{algorithm}
\section{Main Results on Binary Classification}
\label{sec:main_result}

In this section, we present our main result on benign overfitting of a ReLU DNN for binary classification under the lazy training regime.
The data generation process is introduced in \cref{ssec:data}, and the related assumptions that are given in \cref{sec:assudata}.
Our main theory and proof sketch are presented in \cref{sec:theoryclass} and \cref{sec:proofsketchclass}, respectively. We use NTK initialization~\citep{allen2019convergence} in this section, but the main result can be easily extended to more initializations, such as He~\citep{he2015delving} and LeCun~\citep{lecun2012efficient}.

\subsection{Data generation process}
\label{ssec:data}

We consider a standard mixture model setting \cite{JMLR_v22_20_974, frei2022benign} in benign overfitting for binary classification, where a joint distribution $P$ is defined over $(\bm{x}, y) \in \mathbb{R}^{d} \times \left \{ \pm 1 \right \}$ and samples from this distribution can have noisy labels. Following~\citet{frei2022benign}, we first define the clean distribution $\widetilde{P}$ and then define the true distribution $P$ based on $\widetilde{P}$: 

\begin{enumerate}
    \item Sample a clean label $\widetilde{y}$ uniformly at random, $\widetilde{y} \sim \text{Uniform}(\left \{ + 1, -1 \right \})$.
    \item Sample $\bm{z} \sim P_{\text{clust}}$ that satisfy:
    \begin{itemize}
        \item $P_{\text{clust}} = P_{\text{clust}}^{(1)} \times \cdots \times P_{\text{clust}}^{(d)}$ is a product distribution whose marginals are all mean-zero with the sub-Gaussian norm at most one;
        \item $P_{\text{clust}}$ is a $\lambda$-strongly log-concave distribution over $\mathbb{R}^{d}$ for some $\lambda>0$;
        \item For some $\kappa$, it holds that $\mathbb{E}_{\bm{z}\sim P_{\text{clust}}}(\left \| \bm{z} \right \|^2) > \kappa d$.
    \end{itemize}
    \item Generate $\widetilde{\bm{x}} = \bm{z} + \widetilde{y}\bm{\mu}$.
    \item Then, given a noise rate $\eta \in [0, \frac{1}{2})$, $P$ is any distribution over $ \mathbb{R}^{d} \times \left \{ \pm 1 \right \}$ such that the marginal distribution of the features for $P$ and $\widetilde{P}$ coincide, and the total variation distance between the two distributions satisfies $d_{\text{TV}}(\widetilde{P}, P) \leq \eta$. Specifically, $P$ has the same marginal distribution over $\bm{x}$ as $\widetilde{P}$, but a sample $(\bm{x}, y)\sim P$ has a label equal to $\widetilde{y}$ with probability $1-\eta$ and has a label equal to $-\widetilde{y}$ with probability $\eta$. That is, the labels are flipped with $\eta$ ratio.
\end{enumerate}

We denote by $\mathcal{C} \subset [n]$ the set of indices corresponding to samples with clean labels, and $\mathcal{C}'$ as the set of indices corresponding to noisy labels so that $i \in \mathcal{C}'$ implies $(\bm{x}_i, y_i)\sim P$ is such that $y_i = -\widetilde{y}_i$ using the notation above.

\subsection{Assumptions}
\label{sec:assudata}
We make two assumptions about the data distribution.
\begin{assumption}[\citet{du2019gradient, allen2019convergence}]
\label{assumption:distribution_1}
We assume that the data is bounded, i.e. there is a constant $C_{\text{norm}}$ that satisfies $\| \bm x \|_{2} \leq C_{\text{norm}}$.
\end{assumption}

\begin{assumption}
\label{assumption:kernel}
For two different data sample $(\bm{x}_1,\widetilde{y}_1), (\bm{x}_2,\widetilde{y}_2) \sim \widetilde{P}$, The NTK kernel defined in~\cref{eq:NTK} satisfies that:
\begin{equation*}
\begin{split}
    & \mathbb{E}_{(\bm{x}_1,\widetilde{y}_1), (\bm{x}_2,\widetilde{y}_2) \sim \widetilde{P}}\big[K_{\tt NTK}(\bm{x}_1,\bm{x}_2)|\widetilde{y}_1 = \widetilde{y}_2\big] \\
    - & \mathbb{E}_{(\bm{x}_1,\widetilde{y}_1), (\bm{x}_2,\widetilde{y}_2) \sim \widetilde{P}}\big[K_{\tt NTK}(\bm{x}_1,\bm{x}_2)|\widetilde{y}_1 \neq \widetilde{y}_2\big]\\
    \geq & C_{N} > 0\,.
\end{split}
\end{equation*}
\end{assumption}
{\bf Remark:} This assumption states that the NTK value for data points belonging to the same class is larger than that for a different class, in expectation. 
This makes sense in practice, since, as a kernel, the NTK is able to evaluate the similarity of two data points \cite{scholkopf2002learning}: if they are from the same class, the similarity value is large and vice versa.
To verify this assumption, we give an example of the two-layer NTK over a uniform distribution inside a multidimensional sphere such that $C_N = \Theta(1/\sqrt{d})$, refer to \cref{sec:example_of_assumption}.

Besides, we also empirically verify our assumption on MNIST~\citep{726791} with ten digits from 0 to 9. We randomly sample $1,000$ data for each digit and calculate the empirical mean (to approximate the expectation) of the two-layer NTK kernel value over these digit pairs. The experimental result is shown in~\cref{fig:experiemnt}.
We can see that the confusion matrix usually has a larger diagonal element than its non-diagonal element, which implies that the kernel value on the same class is often larger than that of different classes.
This verifies the justification of our assumption.

\begin{figure}
    \centering 
    \includegraphics[width=1.1\linewidth]{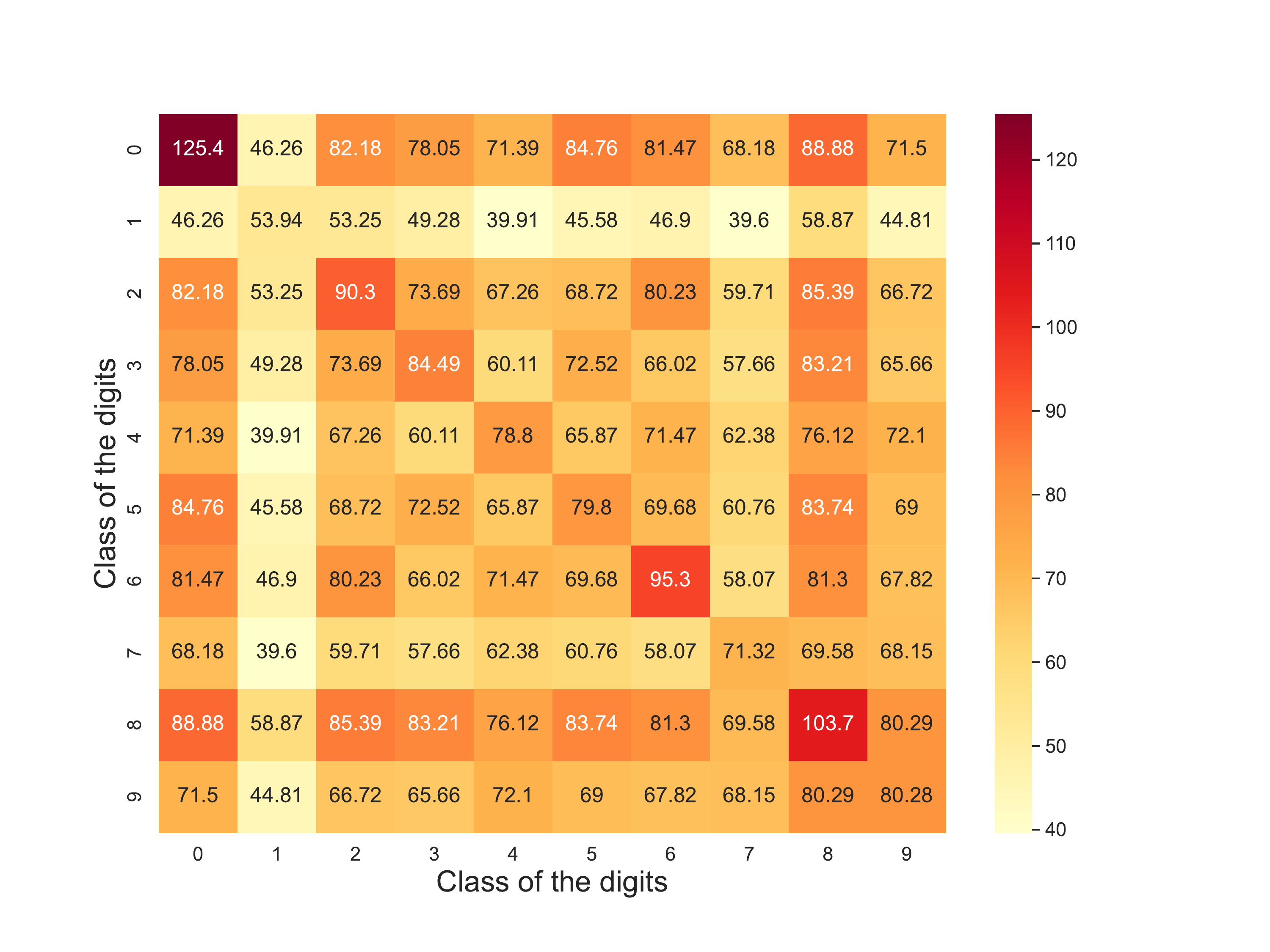}
\caption{Averaged kernel values among 10 classes in MNIST, where a larger kernel value indicates a higher similarity between two data pairs.}
\label{fig:experiemnt}
\end{figure}

\subsection{Theoretical guarantees}
\label{sec:theoryclass}
Based on our assumption, we are ready to present our theoretical result that the test error of a ReLU DNN is close to the Bayes-optimal (noise rate).

\begin{theorem}
\label{thm:main_thm}
    Given a DNN defined by~\cref{eq:deep_network} and trained by~\cref{alg:algorithm_SGD} with a step size 
    $\alpha  \gtrsim L^{-2}(\log m)^{-5/2}$. Then under \cref{assumption:distribution_1} and~\ref{assumption:kernel}, for $\omega \leq \mathcal{O}(L^{-9/2}(\log m)^{-3})$ and $\lambda > 0$, with probability at least $1 - \mathcal{O}(nL^2)\exp(-\Omega(m\omega^{2/3}L))$, we have:
    \begin{equation*}
    \begin{split}
        & \mathbb{P}_{(\bm{x},y)\sim P}(y \neq \mathrm{sgn}(f(\bm{x};\bm{W}^{(n)})))\\
        \leq & \eta + \exp\bigg(-\lambda\Theta\bigg( \frac{n\alpha(1-2\eta)C_{N}}{\text{Lip}_{f(\bm{x};\bm{W}^{(n)})}} \bigg)^2 \bigg)\,,
    \end{split}
    \end{equation*}

    where the $\eta$ is the noise rate defined in~\cref{ssec:data} and $C_N$ is defined in~\cref{assumption:kernel}.
\end{theorem}

{\bf Remark:}~\cref{thm:main_thm} provides the upper bound on the test error rate, including two parts. The first part is the proportion of the wrong labels in the training data.
The second part exponentially decreases with the square of the number of training samples $n$. 
Also, this term is positively correlated with the Lipschitz constants of DNNs after training, which implies that interpolating more smooth functions leads to a faster convergence rate. We take a closer look at this phenomenon in~\cref{sec:kernel}, analyzing how neural networks interpolate target functions in a regression setting. Overall, this bound shows that the models overfit the wrong or noisy data on the training set, but still achieve good generalization error on the testing set. This is consistent with previous work on broader settings of benign overfitting that are not limited to classification problems with label noise. For example, various regression problems~\citep{Bartlett_2020, zou2021benign, chatterji2022foolish, koehler2021uniform, tsigler2020benign}, classification problems of 2-layer networks~\citep{frei2022benign, cao2022benign}, 2-layer and 3-layer NTK networks~\citep{ju2021generalization,ju2022on}.

% Overall, this bound shows that even if the training data contains label noise, deep neural networks still achieve near \textit{Bayes-optimal} error rates on the test set.

Here we discuss the (nearly) zero-training loss and how the Lipschitz constant affects our error bounds.

{\bf SGD can obtain arbitrarily (nearly) zero-training errors on the training set:} A lot of work has shown that deep neural networks trained with SGD can obtain zero training error on the training set and perfectly fit any training label in both classification and regression problems with mean squared loss or logistic loss~\citep{du2018gradient, du2019gradient,NEURIPS2019_ae614c55,zou2020gradient}. These results for empirical loss need to be over-parameterized by the condition that $m = \text{poly}(n, L)$.
In~\cref{sec:optimization}, we provide proof that the loss can be arbitrarily small on the training set under the setting of~\cref{sec:preliminaries}.
This indicates that when the training data has label noise, the neural network will learn all the noise, that is, overfitting. Combined with the bound of the test error rate in ~\cref{thm:main_thm}, we can say that the deep neural network has a benign overfitting phenomenon.

{\bf Lipschitz constant of the deep neural network:}~\cref{thm:main_thm} shows that the convergence rate of the test error rate with the amount of data and is closely related to the Lipschitz constant of the neural network.
The Lipschitz constant of DNNs has been widely studied in \cite{bubeck2021law,wu2021wider,huang2021exploring,nguyen2021tight}.
For example, for ReLU DNNs, if we employ the result of~\citet[Theorem 6.2]{nguyen2021tight}: $\text{Lip}_{f} \lesssim \mathcal{O}\big((2\log m)^{L-1}\big)$, then our bound is 
\begin{equation*}
\begin{split}
   & \mathbb{P}_{(\bm{x},y)\sim P}(y \neq \mathrm{sgn}(f(\bm{x};\bm{W}^{(n)})))   \\
    & \qquad \lesssim \eta + \exp\left(-\lambda\Theta\Big( \frac{n}{(2\log m)^{L-1}} \Big)^2 \right)\,,
\end{split}
\end{equation*}
 which leads to a better convergence rate on generalization than the two-layer result~\citet{frei2022benign}.

\subsection{Proof sketch of \cref{thm:main_thm}}
\label{sec:proofsketchclass}
Let us first introduce a few relevant lemmas.

The first Lemma will follow by establishing a lower bound for the expected normalized margin on clean points, $\mathbb{E}_{(\bm{x},\widetilde{y})\sim \widetilde{P}}[\widetilde{y}f(\bm{x};\bm{W})]/\text{Lip}_{f(\bm{x};\bm{W}^{(t)})}$.
\begin{restatable}{lemma}{lipschitzconcentration}\label{lemma:lipschitz_concentration}
Given a DNN defined by~\cref{eq:deep_network} and trained by~\cref{alg:algorithm_SGD}. For any $t\geq 0$, assuming $\mathbb{E}_{(\bm{x},\widetilde{y})\sim \widetilde{P}}[\widetilde{y}f(\bm{x};\bm{W}^{(t)})] \geq 0$, then we have:
    \begin{equation*}
    \begin{split}
        & \mathbb{P}_{(\bm{x},y)\sim P}(y \neq \mathrm{sgn}(f(\bm{x};\bm{W}^{(t)}))) \\
        \leq & \eta + \exp\bigg(-\frac{\lambda}{4} \bigg( \frac{\mathbb{E}_{(\bm{x},\widetilde{y})\sim \widetilde{P}}[\widetilde{y}f(\bm{x};\bm{W}^{(t)})]}{\text{Lip}_{f(\bm{x};\bm{W}^{(t)})}} \bigg)^2 \bigg)\,.
    \end{split}
    \end{equation*}
\end{restatable}

We next introduce some structural results concerning the neural network optimization objective. The following lemma states that near initialization, the neural network function is almost linear in terms of its weights.

\begin{restatable}{lemma}{almostlinear}\label{lemma:almost_linear}
Let $\bm{W}, \bm{W}' \in \mathcal{B} (\bm{W}^{(0)},\omega )$  with $\omega = \mathcal{O} (L^{-9/2}(\log m)^{-3} )$, for any $\bm{x}\in \mathbb{R}^d$ that satisfy~\cref{assumption:distribution_1}, with probability at least $1- \exp(-\Omega(m\omega^2\log m))-\mathcal{O}(nL^2)\exp(-\Omega(m\omega^{2/3}L))$, we have:

\begin{equation*}
\begin{split}
    & |f(\bm{x}; \bm{W}) - f(\bm{x}; \bm{W}') - \left \langle \nabla f(\bm{x}; \bm{W}' ), \bm{W} - \bm{W}'\right \rangle| \\
    \leq & \mathcal{O}(\sqrt{\omega^2 L^3 m \log m})\sum_{l=1}^{L-1}\left \| \bm{W}_l-\bm{W}_l' \right \|_2\,.
\end{split}
\end{equation*}
\end{restatable}

The following lemma describes the change of $yf(\bm{x}; \bm{W}^{(t+1)})$ from time $t$ to $t+1$.

\begin{restatable}{lemma}{changeeverystep}\label{lemma:change_every_step}
Given a DNN defined by~\cref{eq:deep_network} and trained by~\cref{alg:algorithm_SGD}. For any $t\geq 0$ and $(\bm{x},\widetilde{y})\sim \widetilde{P}$ that satisfy~\cref{assumption:distribution_1},  with $\omega = \mathcal{O} (L^{-9/2}(\log m)^{-3} )$, with probability at least $1- \exp(-\Omega(m\omega^2\log m))-\mathcal{O}(nL^2)\exp(-\Omega(m\omega^{2/3}L))$, we have:
\begin{equation*}
\begin{split}
    & \widetilde{y}  [f(\bm{x}; \bm{W}^{(t+1)}) - f(\bm{x}; \bm{W}^{(t)})]\\
    \geq & \alpha g_i^{(t)} \left \langle \widetilde{y} \nabla f(\bm{x}; \bm{W}^{(t)} ), y_i \nabla f(\bm{x}_i; \bm{W}^{(t)} ) \right \rangle\\
    - & \mathcal{O}(\sqrt{\omega^2 L^3 m \log m})\sum_{l=1}^{L-1}\left \| \bm{W}^{(t+1)}_l - \bm{W}^{(t)}_{l} \right \|_2\,,
\end{split}
\end{equation*}
where $(\bm {x}_i, y_i )$ is the random selected training sample at step $t + 1$.
\end{restatable}

Based on the previous lemmas, we can now derive a lower bound on the normalized margin. Note that this lower bound on the normalized margin in conjunction with~\cref{lemma:lipschitz_concentration} results in the test error bound for the main theorem.

\begin{restatable}{lemma}{normalizedmarginbound}\label{lemma:normalized_margin}
Let us define a DNN using \cref{eq:deep_network} and trained by~\cref{alg:algorithm_SGD} with a step size $\alpha  \gtrsim L^{-2}(\log m)^{-5/2}$. Then under~\cref{assumption:distribution_1} and~\ref{assumption:kernel}, for any $t\geq 0$, $\omega \leq \mathcal{O}(L^{-9/2}(\log m)^{-3})$, with probability at least $1- \exp(-\Omega(m\omega^2\log m)) - \mathcal{O}(nL^2)\exp(-\Omega(m\omega^{2/3}L))$, we have:
\begin{equation*}
    \mathbb{E}_{(\bm{x},\widetilde{y})\sim \widetilde{P}}[\widetilde{y}f(\bm{x};\bm{W}^{(t)})] \geq \Theta\big(t\alpha(1-2\eta)C_{N}\big)\,.
\end{equation*}
\end{restatable}

Now, we can prove ~\cref{thm:main_thm}.

\begin{proof}
    According to~\cref{lemma:lipschitz_concentration}, choosing $t := n$, we have:

    \begin{equation*}
    \begin{split}
        & \mathbb{P}_{(\bm{x},y)\sim P}(y \neq \mathrm{sgn}(f(\bm{x};\bm{W}^{(n)}))) \\
        \leq & \eta + \exp\bigg(-\frac{\lambda}{4} \bigg( \frac{\mathbb{E}_{(\bm{x},\widetilde{y})\sim \widetilde{P}}[\widetilde{y}f(\bm{x};\bm{W}^{(n)})]}{\text{Lip}_{f(\bm{x};\bm{W}^{(n)})}} \bigg)^2 \bigg)\,.
    \end{split}
    \end{equation*}

    Then, by~\cref{lemma:normalized_margin}, choosing $t := n$ and $\alpha  \gtrsim L^{-2}(\log m)^{-5/2}$, for $\omega \leq \mathcal{O}(L^{-9/2}(\log m)^{-3})$, with probability at least $1 - \mathcal{O}(nL^2)\exp(-\Omega(m\omega^{2/3}L))$, we have:

    \begin{equation*}
    \mathbb{E}_{(\bm{x},\widetilde{y})\sim \widetilde{P}}[\widetilde{y}f(\bm{x};\bm{W}^{(n)})] \geq \Theta\big(n\alpha(1-2\eta)C_{N}\big)\,.
    \end{equation*}
    
    Combine the results, we have:
    \begin{equation*}
    \begin{split}
        & \mathbb{P}_{(\bm{x},y)\sim P}(y \neq \mathrm{sgn}(f(\bm{x};\bm{W}^{(n)})))\\
        \leq & \eta + \exp\bigg(-\lambda \Theta\bigg( \frac{n\alpha(1-2\eta)C_{N}}{\text{Lip}_{f(\bm{x};\bm{W}^{(n)})}} \bigg)^2 \bigg)\,.
    \end{split}
    \end{equation*}
    
\end{proof}
\section{Interpolating Smooth Function by NTK}
\label{sec:kernel}

In this section, we take a closer look at the phenomenon of the relationship between Lipschitz constants of DNNs and convergence rate in \cref{thm:main_thm}, and accordingly analyze how neural networks interpolate smooth ground-truth functions in a regression setting from an approximation theory view \cite{cucker2007learning}. In this section, we will also follow the NTK initialization~\citep{allen2019convergence}. For other different initialization~\citep{cao2019generalization, arora2019exact}, similar conclusions will apply.

To be specific, let $X \subseteq \mathbb{R}^{d}$ be an input space, and $Y \subseteq \mathbb{R}$ be the output space, $f_{\rho}:X \to Y $ be the ground-truth function, that is smooth in RKHS, described by the source condition in \cref{ssec:kernel_ps}.
We assume that the data $(\bm x, y)$ is sampled from an unknown distribution $\rho$, and $\rho_X$ is the marginal distribution of $\rho$ over $X$.
 The label is generated through $y = f_{\rho}(\bm{x})+\epsilon$, where $\epsilon$ is the noise. 
Accordingly, denote $L_{\rho_X}^2$ as the $\rho_X$ weighted $L^2$-space and its norm $\left \| f \right \|_{L_{\rho_X}^2}^{2}=\int_{X}\left | f(\bm{x}) \right |^2d\rho_{X}(\bm{x})$, we are interested in the excess risk $\| f(\bm x; \bm W^{(t)}) - f_{\rho} \|_{L_{\rho_X}^2}^{2}$, which describes how neural networks interpolate/approximate a smooth ground-truth function in a certain space \cite{cucker2007learning,Bach2017}. In this section, we use the standard NTK network and initialization, which is equivalent to~\citet{arora2019exact} using the initialization with standard normal distribution together with the scale factor after each layer for the training dynamics.

\subsection{Assumptions}
\label{ssec:kernel_ps}

We make the following assumptions:

\begin{assumption}[High dimensionality~\citep{liang2020just,liu2021kernel}]
\label{assumption:distribution_high_d}
There exists universal constants $c_1, c_2 \in (0, \infty)$ such that $c_1 \leq \frac{d}{n} \leq c_2$.
\end{assumption}

\begin{assumption}[Noise condition~\citep{liang2020just,liu2021kernel}]
\label{assumption:distribution_noise}
There exists a $\sigma_{\epsilon} > 0$ such that $\mathbb{E}[(f_{\rho}(\bm{x})-y)^2|\bm{x}]\leq \sigma_{\epsilon}^2$, almost surely.
\end{assumption}

\begin{assumption}[\citet{geifman2020similarity, chen2021deep}]
\label{assumption:distribution_sphere}
We assume that $\bm{x}_i, \forall i \in [n]$ are i.i.d. sampled from a uniform distribution on the $d$-dimensional unit sphere. i.e. $\bm{x} \sim \text{Unif}(\mathbb{S}^{d-1}(1)), \ \mathbb{S}^{d-1}(1):=\left \{ \bm{x}\in \mathbb{R}^d | \left \| \bm{x} \right \|_2 = 1 \right \} $.
\end{assumption}

{\bf Remark:} 
The i.i.d unit sphere data assumption implies that the data $\bm{x}$ is isotropic \emph{asymptotically} under our high-dimensional setting \cite{wainwright2019high}, i.e., $\mathbb{E} [\bm{xx}^{\!\top}] = \mathbb{I}_{d}/d$. 
In fact, there is an alternative way in our proof by directly assuming $\bm x$ is sub-Gaussian and $\mathbb{E} [\bm{xx}^{\!\top}] = \mathbb{I}_{d}/d$.

\begin{assumption}[Existence of $f_{\rho}$]
\label{assumption:distribution_5}
We assume the ground-truth function $f_{\rho} \in \mathcal{H}_{\text{NTK}}$, where $\mathcal{H}_{\text{NTK}}$ is the RKHS associated with the limiting NTK kernel. 
\end{assumption}
{\bf Remark:} This is a standard assumption in learning theory by assuming that the ground-truth function $f_{\rho}$ is indeed realizable \citep{cucker2007learning,rudi2017generalization,liu2021kernel}.
This assumption is a special case of the source condition \cite{cucker2007learning} by taking certain values and can be easily extended to non-RKHS spaces or teacher-student settings~\citep{hinton2015distilling}. For ease of analysis, we directly assume the ground-truth function in an RKHS.

\subsection{Kernel regression estimator}

Let $\bm{X} \in \mathbb{R}^{n \times d}$ be a matrix, each column of which is the input of one training sample, $\bm{\epsilon} \in \mathbb{R}^{n \times 1}$ be the noise in the output of training data. The empirical risk minimization (ERM) is defined with the squared loss:

\begin{equation}
    \hat{f}_{\bm{z}} = \arg \min_{f \in \mathcal{F}} \left \{ \frac{1}{2n}\sum_{i=1}^{n}(f(\bm x_i) - y_i)^2 \right \}\,,
\label{eq:ERM}
\end{equation}
where the hypothesis space $\mathcal{F}$ can be defined properly.
For example, if $\mathcal{F}$ is a RKHS $\mathcal{H}_{\bm{K}}$, \cref{eq:ERM} is formulated as a kernel regression. Denoting that $\bm{K}_{\tt ker}(\bm{x}, \bm{X}) = [\bm{K}_{\tt ker}(\bm{x}, \bm{x}_1), \bm{K}_{\tt ker}(\bm{x}, \bm{x}_1), \dots, \bm{K}_{\tt ker}(\bm{x}, \bm{x}_n) ]^{\top} \in \mathbb{R}^n$, the closed form of the kernel regression estimator to \cref{eq:ERM} is given by:
\begin{equation*}
    f_{\tt ker} = \bm{K}_{\tt ker}(\bm{x}, \bm{X})^{\top} \bm{K}_{\tt ker}^{-1}\bm{y}\,.
\end{equation*}

If we use a neural network as in~\cref{eq:deep_network} to solve \cref{eq:ERM}, the corresponding hypothesis space $\mathcal{F}_{\tt nn}$ is:
\begin{equation*}
\begin{split}   
\mathcal{F}_{\tt nn} & := \bigg\{f(\bm x; \bm W)~\mbox{admits Eq.~\eqref{eq:deep_network}}:  \bm{x} \sim \text{Unif}(\mathbb{S}^{d-1}(1)),
\\
& \bm{W} \in \mathbb{R}^{m\times d} \times (\mathbb{R}^{m\times m})^{L-2}\times \mathbb{R}^{1 \times m} \bigg\} \,,
\end{split}
\end{equation*}
which implies:
\begin{equation*}
    f_{\tt nn} = \arg \min_{f \in \mathcal{F}_{\tt nn}} \left \{ \frac{1}{2n}\sum_{i=1}^{n}(f(\bm x_i; \bm{W}) - y_i)^2 \right \}\,.
\end{equation*}

In addition to the neural tangent kernel mentioned earlier~\cref{eq:NTK}, we will present some examples of the positive definite kernels to be studied in this paper.

{\bf Dot product kernel}~\citep{ghosh2022the}:
The dot product kernels have the following forms:

\begin{equation*}
    K_{\tt dot}(\bm{x}, \widetilde{\bm{x}}) = k(\left \langle \bm{x}, \widetilde{\bm{x}} \right \rangle), \quad \forall \bm{x}, \widetilde{\bm{x}} \in \mathbb{S}^{d-1}(1)\,,
\end{equation*}

for some function $k:[-1,1] \to \mathbb{R}$.

{\bf Laplace kernel}~\citep{geifman2020similarity}: The Laplace kernel is defined as:
\begin{equation*}
    K_{\tt Laplace}(\bm{x}, \widetilde{\bm{x}}) = e^{-c\left \| \bm{x}-\widetilde{\bm{x}} \right \|_2 }, \quad c > 0\,.
\end{equation*}
According to~\cref{assumption:distribution_sphere}, we have:
\begin{equation}
    K_{\tt Laplace}(\bm{x}, \widetilde{\bm{x}}) = e^{-c\sqrt{2(1-\bm{x}^{\top}\widetilde{\bm{x}})}} = e^{-\widetilde{c}\sqrt{1-u}}\triangleq K_{\tt dot}(u)\,,
\label{eq:laplace_is_dot}
\end{equation}
where $u = \left \langle \bm{x}, \widetilde{\bm{x}} \right \rangle$.

\subsection{The minimum eigenvalue of NTK matrix under the high dimensional setting}
\label{ssec:min_eigen}

We are now ready to state the main result of a deep over-parameterized NTK network. We first provide the lower bounds of the minimum eigenvalue of NTK under the high dimensional setting.

Recall that the Neural Tangent Kernel (NTK)~\citep{jacot2018neural} matrix of neural network $f$ is defined in~\cref{eq:NTK}. When we focus on the infinite-width setting ($m\to \infty$), the NTK matrix for a neural network~\cref{eq:deep_network} is derived by the following regular chain rule.

\begin{lemma}[Adapted from Lemma 3.1 in~\citet{nguyen2021tight}]
\label{lemma:NTK_matrix_recursive_form}
For any $l \in [3, L]$ and $s \in [2, L]$, denote
\begin{equation*}
\small
    \begin{split}
        & \bm{G}^{(1)}=\bm{XX}^\top\,,\\
        & \bm{G}^{(2)}=2\mathbb{E}_{\bm{w} \sim \mathcal N(\bm 0,\mathbb{I}_{d})}[\sigma_1(\bm{Xw})\sigma_1(\bm{Xw})^\top]\,,\\
        & \bm{G}^{(l)}=2\mathbb{E}_{\bm{w} \sim \mathcal N(\bm 0,\mathbb{I}_{N})}[\sigma_{l-1}(\sqrt{\bm{G}^{(l-1)}} \bm{w})\sigma_{l-1}(\sqrt{\bm{G}^{(l-1)}} \bm{w})^\top]\,,\\
        & \dot{\bm{G}}^{(s)} = 2\mathbb{E}_{\bm{w} \sim \mathcal N(\bm{0},\mathbb{I}_{N})}[{\sigma}'_{s-1}(\sqrt{\bm{G}^{(s-1)}} \bm{w}){\sigma}'_{s-1}(\sqrt{\bm{G}^{(s-1)}} \bm{w})^\top]\,.
    \end{split}
\end{equation*}

Then, the NTK for a $L$-layer neural network defined in~\cref{eq:deep_network} can be written as
\begin{equation*}
    \bm{K}_{\tt NTK}=\bm{G}^{(L)} + \sum_{l=1}^{L-1}\bm{G}^{(l)}\circ \dot{\bm{G}}^{(l+1)} \circ \dot{\bm{G}}^{(l+2)}\circ \cdots \circ \dot{\bm{G}}^{(L)}\,,
\end{equation*}

where $\circ$ represents the element-wise Hadamard product.
\end{lemma}

Based on the formulation of NTK, we are ready to present the estimation of the minimum eigenvalue of NTK.

\begin{theorem}[Minimum eigenvalue of NTK matrix]
\label{thm:min_eigen_NTK}
For a DNN defined by~\cref{eq:deep_network}, let $\bm{K}_{\tt NTK}$ be the limiting NTK recursively defined in~\cref{lemma:NTK_matrix_recursive_form} and let $\lambda_0$ be the minimum eigenvalue of $\bm{K}_{\tt NTK}$. Then, under~\cref{assumption:distribution_high_d} and~\ref{assumption:distribution_sphere}, with probability at least $1-2e^{-n}$, we obtain that:
\begin{equation*}
\lambda_0 \geq \begin{cases}
  & 2\mu_{1}^{2}\frac{n}{d}\bigg(\frac{3}{4}-\frac{c}{4}\sqrt{\frac{d}{n}}\bigg)^2, \quad \text{if} \quad n \geq d, \\
  & 2\mu_{1}^{2}\frac{n}{d}\bigg(\sqrt{\frac{d}{n}}-\frac{c+6}{4}\bigg)^2, \quad \text{if} \quad n < d\,,
\end{cases}
\end{equation*}
where we have an absolute constant $c = 2^{3.5}\sqrt{\log(9)} \approx 16.77$ and $\mu_1$ is the $1$-st Hermite coefficient of the ReLU activation function.
\end{theorem}

{\bf Remark:} 
This theorem provides the upper bound of the minimum eigenvalue of the NTK matrix of the infinite-width neural network under the high-dimensional setting and can be easily extended to the finite-width setting.
Note that our result under the high dimensional setting is different from previous work under the fixed $d$ setting in~\citet{zhu2022generalization}. If we fix $d$ and vary $n$ from small to large, there exists a phase transition when $n$ increases, see \cref{tab:trend}. 
If $n \ll d$, the lower bound of the minimum eigenvalue of NTK $\lambda_0 \geqslant \Omega(1)$, then when $n$ increases, we can see that $\lambda_0$ decreases to a bottom and then increases until $n:=d$. In the $n \geq d$ regime, there exists a similar trend to that of the $n \leq d$ regime: firstly decreasing and then increasing. When $n \gg d$, we have $\lambda_0 \geqslant \Omega(n)$.  
\begin{table}[H]
\centering
\caption{The trend of the bound with respect to $n$ under different range of $n$ values and a fixed $d$.} 
\begin{tabular}{l@{\hspace{0.25cm}} c@{\hspace{0.25cm}}c@{\hspace{0.25cm}}c@{\hspace{0.2cm}} c} 
    \toprule
    Range of $n$ & Trend w.r.t. $n$ & Limit bound \\
    \midrule
    $n \ll d$ & - & $2\mu_{1}^{2}$ \\
    \midrule
    $0 \leq n \leq (\frac{4}{c+6})^2 d$ & $\searrow$ & - \\
    \midrule
    $(\frac{4}{c+6})^2 \leq n \leq d$ & $\nearrow$ & - \\
    \midrule
    $d \leq n \leq \frac{c^2}{9}d$ & $\searrow$ & - \\
    \midrule
    $\frac{c^2}{9}d \leq n$ & $\nearrow$ & - \\
    \midrule
    $n \gg d$ & - & $\frac{9}{8}\mu_{1}^{2}\frac{n}{d}$ \\
    \midrule
\end{tabular}
\label{tab:trend}
\end{table}

\subsection{Generalization error bound}
\label{ssec:kernel_results}

Based on the aforementioned upper and lower bounds of the minimum eigenvalue of NTK under the high dimensional setting, we establish the relationship between the minimum eigenvalue of NTK and the generalization error of DNNs. We provide a bound on the norm of the difference between the network output and ground truth function under the weighted $L^2_{\rho_X}$ space.

\begin{theorem}[An upper bound on the generalization error for deep over-parameterized NTK network]
\label{thm:kernel}

Let $\theta \in (0,1/2]$, $\delta$ and $c$ are some non-negative constants, the ground-truth function $f_{\rho}$ lies in a RKHS by~\cref{assumption:distribution_5} and $d$ large enough, under~\cref{assumption:distribution_high_d},~\ref{assumption:distribution_noise} and~\ref{assumption:distribution_sphere}, suppose that, $\omega \leq poly(1/n, \lambda_0, 1/L, 1/\log(m), \epsilon, 1/\log(1/\delta'), \kappa)$, $m\geq poly(1/\omega)$ and $\kappa = \mathcal{O}(\frac{\epsilon }{\log(n/\delta')})$. then for any given $\varepsilon > 0$, with high probability, we have:

\begin{equation*}
\begin{split}
     \mathbb{E} & \left \| f_{\tt nn} -f_{\rho} \right \|_{L_{\rho_X}^2}^{2} 
    \lesssim \mathcal{O} \bigg(n^{-\theta}\log^4 (\frac{2}{\delta})  + \frac{\sigma_{\epsilon}^2}{d}\mathcal{N}_{\widetilde{\bm{X}}} \\
    + & \frac{\sigma_{\epsilon}^2 \log ^{2+4\varepsilon}d}{d^{4\theta-1}} + \epsilon^2 + \frac{n}{\lambda_0^2}\omega^{2/3}L^5 m \log m+\frac{n^3}{\lambda_0^6\kappa^2}\bigg)\,,
\end{split}
\end{equation*}

where the $\lambda_0$ satisfies \cref{thm:min_eigen_NTK} and the effective dimension $\mathcal{N}_{\widetilde{\bm{X}}}$ is defined as:
\begin{equation*}
    \mathcal{N}_{\widetilde{\bm{X}}} := \sum_{i=0}^{n-1} \frac{\lambda_i(\widetilde{\bm{X}})}{(\lambda_i(\widetilde{\bm{X}} +\gamma ))^2}\,,
\end{equation*}
with $\widetilde{\bm{X}} := \beta \bm{X}\bm{X}^{\top}/d + \alpha \bm{1}\bm{1}^{\top}$ for some non-negative constants $\alpha$, $\beta$, $\gamma$.

\end{theorem}

{\bf Remark:} 
% Our theorem builds a connection between DNNs and kernel methods in benign overfitting. Under refined assumptions, e.g., source condition, capacity condition \citep{cucker2007learning}, we can achieve $\theta=1$ for a better convergence rate.

This theorem builds a connection between DNNs and kernel methods in benign overfitting and gives the upper bound of the generalization error of the NTK network in the high-dimensional setting. To be specific,the first term is the upper bound of the bias of NTK regression, which decreases as the number of data increases. The second and the third terms jointly form the upper bound of variance of NTK regression, which is mainly affected by the effective dimension (eigenvalue decay) of the data. The fourth term is the error introduced by the initialization of the NTK neural network. The fifth and the sixth term reflect the difference between the finite-width NTK network and the infinite-width NTK network (neural tangent kernel regression), which decreases with the increase of the minimum eigenvalue of the NTK network we provide in~\cref{thm:min_eigen_NTK}.

Under refined assumptions, e.g., source condition, capacity condition \citep{cucker2007learning}, we can achieve $\theta=1$ for a better convergence rate. Regarding the convergence properties, we need to make the following discussion.

The three non-negative constants $\alpha$, $\beta$, and $\gamma$ are related to the linearization of the kernel matrix in the high dimension setting, refer to \citet{liu2021kernel} for details. Here we give the following discussion on $\mathcal{N}_{\widetilde{\bm{X}}}$ under three typical eigenvalue decay of $\bm X \bm X^{\!\top}$ cases, and then discuss our generalization bound.
Note that when $n > d$, the sample matrix $\bm X \bm X^{\!\top}$ has at most $d$ eigenvalues, so we can directly have $\mathcal{N}_{\widetilde{\bm{X}}} \leq \mathcal{O}(d)$.
Accordingly, here we present the results on the $n<d$ case.

\begin{itemize}
    \item {\bf Harmonic decay:} $\lambda_i(\widetilde{\bm{X}}) \propto n/i, \forall i \in \left \{ 1,2,\dots, r_{\star} \right \}$ and $\lambda_i(\widetilde{\bm{X}}) = 0, \forall i \in \left \{ r_{\star}+1, \dots, n \right \}$.

    We have: $\mathcal{N}_{\bm{X}} = \mathcal{O}(n)$, then the term $\frac{\sigma_{\epsilon}^2}{d}\mathcal{N}_{\widetilde{\bm{X}}}  \leq \mathcal{O}(\frac{\sigma_{\epsilon}^2}{d}n)$.
    \item {\bf Polynomial decay:} $\lambda_i(\widetilde{\bm{X}}) \propto ni^{-2a}$ with $a>1/2, \forall i \in \left \{ 1,2,\dots, r_{\star} \right \}$ and $\lambda_i(\widetilde{\bm{X}}) = 0, \forall i \in \left \{ r_{\star}+1, \dots, n \right \}$.

    We have: $\mathcal{N}_{\bm{X}} = \mathcal{O}(n^{1/2a})$, then the term $\frac{\sigma_{\epsilon}^2}{d}\mathcal{N}_{\widetilde{\bm{X}}} \leq  \mathcal{O}(\frac{\sigma_{\epsilon}^2}{d}n^{1/2a}) \leq \mathcal{O}(\frac{\sigma_{\epsilon}^2}{d}n)$.

    \item {\bf Exponential decay:} $\lambda_i(\widetilde{\bm{X}}) \propto ne^{-ai}$ with $a>0, \forall i \in \left \{ 1,2,\dots, r_{\star} \right \}$ and $\lambda_i(\widetilde{\bm{X}}) = 0, \forall i \in \left \{ r_{\star}+1, \dots, n \right \}$.

    We have: $\mathcal{N}_{\widetilde{\bm{X}}} = \frac{1}{a}\big(\frac{1}{\gamma + n\exp(-a(r_{\star}+1))}-\frac{1}{\gamma+n\exp(-a)}\big)$, then the term $\frac{\sigma_{\epsilon}^2}{d}\mathcal{N}_{\widetilde{\bm{X}}} \leq \mathcal{O}(\frac{\sigma_{\epsilon}^2}{d}\frac{e^{ar_{\star}}}{n})$.
\end{itemize}

Based on our discussion on the eigenvalue decay, we are ready to discuss our generalization bound in \cref{thm:kernel}.
When $n,d$ are comparably large enough, e.g., $n \geq \frac{c^2}{9} d$, in this case, we have $\mathcal{N}_{\widetilde{\bm{X}}} \leq \mathcal{O}(d)$,  according to~\cref{thm:min_eigen_NTK} for $\lambda_0$, three terms $n^{-\theta}\log^4 (\frac{2}{\delta})$, $\frac{n}{\lambda_0^2}\omega^{2/3}L^5 m \log m$ and $\frac{n^3}{\lambda_0^6\kappa^2}$ convergence to $0$. 
The term $\frac{\sigma_{\epsilon}^2 \log ^{2+4\varepsilon}d}{d^{4\theta-1}}$ also converges to $0$ for a large enough $d$.
Accordingly, for large enough $n$ and $d$, we have 
\begin{equation*}
     \mathbb{E} \left \| f_{\tt nn} -f_{\rho} \right \|_{L_{\rho_X}^2}^{2} 
    \lesssim \mathcal{O} \big(\sigma_{\epsilon}^2 + \epsilon^2\big)\,, w.h.p\,,
\end{equation*}
which show that the bound only depends on the noise and random initialization term, and thus coincide with previous work on benign overfitting~\citep{frei2022benign, cao2022benign, ju2022on, arora2019exact}.

Besides, we can also find that the phase transition exists in the minimum eigenvalue $\lambda_0$ and the effective dimension $\mathcal{N}_{\widetilde{\bm{X}}}$ from $n<d$ and $n>d$.
This also leads to a phase transition on the excess risk.
Roughly speaking, the excess risk firstly increases with $n$ until $n:=d$ and then decreases with $n$ when $n > d$.
\section{Conclusion and limitations}
\label{sec:conclusion}

In this work, we present a theoretical analysis of benign overfitting for deep ReLU NNs. 
For binary classification, our results demonstrate that DNNs under the lazy training regime obtain the \textit{Bayes-optimal} test error with a better convergence rate than \citet{frei2022benign}.
For regression, our results exhibit a phase transition on the excess risk from $n<d$ to $n>d$, of which the excess risk converges to a constant order $\mathcal{O}(1)$ that only depends on label noise and initialization noise. 
The above two results theoretically validate the benign overfitting of DNNs.

We need to mention that, our results are only applicable to lazy training regimes and appear difficult to be extended to the non-lazy training regime, commonly used in practice. This is because DNN cannot be linearly approximated well under the non-lazy training regime. We leave this as a future work. Besides, an interesting direction is, extending our data-generating distribution assumption from log-concave distribution to a general one, as we require it to ensure the output of neural networks is sub-Gaussian.

\section*{Acknowledgements}
\label{sec:acks}

We are thankful to the reviewers for providing constructive feedback. This work was supported by Hasler Foundation Program: Hasler Responsible AI (project number 21043). This work was supported by SNF project – Deep Optimisation of the Swiss National Science Foundation (SNSF) under grant number 200021\_205011. This work was supported by Zeiss. This project has received funding from the European Research Council (ERC) under the European Union's Horizon 2020 research and innovation programme (grant agreement n° 725594 - time-data).

\bibliography{literature}
\bibliographystyle{abbrvnat}
\newpage
\appendix
\onecolumn
\section*{Appendix introduction} 
\label{sec:appendix_intro}

The Appendix is organized as follows:
\begin{itemize}
    \item In~\cref{sec:symbols_and_notations}, we state the symbols and notation used in this paper.
    \item In~\cref{sec:example_of_assumption}, we provide a example to verify the~\cref{assumption:kernel}.
    \item In~\cref{sec:proof_thm1}, we provide the proof for the lemmas in~\cref{sec:proofsketchclass}.
    \item In~\cref{sec:optimization}, we provide the theorem and its proof of the optimization result for the classification problem.
    \item In~\cref{sec:proofthm2}, we provide the proof for the~\cref{thm:min_eigen_NTK}.
    \item In~\cref{sec:proofsec4}, we provide the proof for the~\cref{thm:kernel}.
\end{itemize}
\newpage

\section{Symbols and Notation}
\label{sec:symbols_and_notations}

In the paper, vectors are indicated with bold small letters and matrices with bold capital letters. To facilitate the understanding of our work, we include some core symbols and notation in \cref{table:symbols_and_notations}. 

\begin{table}[ht]
\caption{Core symbols and notations used in this project.}
\label{table:symbols_and_notations}
\footnotesize
\centering
\begin{tabular}{c | c | c}
\toprule
Symbol & Dimension(s) & Definition \\
\midrule
$\mathcal{N}(\mu,\sigma) $ & - & Gaussian distribution of mean $\mu$ and variance $\sigma$ \\
$\mathcal{B} (\bm{W}, \cdot )$ & - & Neighborhood of matrix $\bm{W}$\\
\midrule
$\lambda(\bm{M})$ & - & Eigenvalues of matrices $\bm{M}$ \\
$\lambda_{\min}(\bm{M})$ & - & Minimum eigenvalue of matrices $\bm{M}$ \\
$\lambda_0$ & - & Minimum eigenvalue NTK matrix \\
\midrule
$\phi(x) = \max(0, x)$ & - & ReLU activation function for scalar \\
$\phi(\bm{v}) = (\phi(v_1), \dots, \phi(v_m))$ & - & ReLU activation function for vectors \\
$1\left \{A\right \}$ & - & Indicator function for event $A$\\
$\mathrm{sgn}(\cdot)$ & - & Sign function\\
$\text{Lip}_{(\cdot)}$ & - & Lipschitz constant of a function\\
\midrule
$n$ & - & Size of the dataset \\
$d$ & - & Input size of the network \\
$L$ & - & Depth of the network \\
$m$ & - & Width of intermediate layer\\
\midrule
$\bm{x}_i$ & $\realnum^{d}$ & The $i$-th data point \\
$y_i$ & $\{\pm 1\}$ & The $i$-th clean label \\
$\widetilde{y}_i$ & $\{\pm 1\}$ & The $i$-th training label \\
$P$ & - & Clean data distribution that $(\bm{x}_i, y_i) \sim P$ \\
$\widetilde{P}$ & - & Training data distribution that $(\bm{x}_i, \widetilde{y}_i) \sim \widetilde{P}$ \\
$P_{\text{clust}}$ & - & Cluster distribution for generate data \\
$\mathcal{C}$ & - & A subset of training data for clean labels \\
$\mathcal{C}'$ & - & A subset of training data for noisy labels \\
\midrule
$\alpha$ & - & Step size of SGD \\
$\eta$ & - & Noise rate \\
$\omega$ & - & Lazy training rate \\
$\lambda$ & - & Strongly log-concave rate of distribution $P_{\text{clust}}$ \\
\midrule
$\ell$ & - &  Logistic loss function \\
$\hat{L}$ & - & Empirical risks \\
$g_i^{(t)}$ & - & Value of $g$ for input $\bm{x}_i$ at time $t$, where $g(z) := -\ell'(z)$ \\
$f_i^{(t)}$ & - & Output of neural network for input $\bm{x}_i$ at time $t$ \\
\midrule
$f_{\rho}$ & - & Ground-truth function \\
$\bm{K}_{\tt NTK}$, $\bm{K}_{\tt Laplace}$, $\bm{K}_{\tt dot}$ & $\realnum^{n \times n}$ & Three different kernel matrices\\
$\mathcal{H}_{\tt NTK}$, $\mathcal{H}_{\tt Laplace}$, $\mathcal{H}_{\tt dot}$ & - & RKHS of the kernel\\
\midrule
$\bm{W}_1$ & $\realnum^{m \times d}$ & Weight matrix for the input layer \\
$\bm{W}_l$ & $\realnum^{m \times m}$ & Weight matrix for the $l$-th hidden layer \\
$\bm{W}_L$ & $\realnum^{1 \times m}$ & Weight matrix for
the output layer \\
$\bm{h}_{i,l}$ & $\realnum^{m}$ & The $l$-th layer activation for input $\bm{x}_i$\\
$\bm{D}_{i,l}$ & $\realnum^{m \times m}$ & Diagonal sign matrix of $l$-th layer input $\bm{x}_i$\\
\midrule
\end{tabular}
\end{table}

\section{A example of~\cref{assumption:kernel}}
\label{sec:example_of_assumption}

\begin{Proposition}
For two different data samples $\bm{x}_1, \bm{x}_2 \in \mathbb{R}^{d} \sim \widetilde{P}:=\text{Unif}(\mathbb{S}^{d-1}(C_{\text{norm}}))$, $y=\begin{cases}
1  & \text{ if } x_i > 0,\forall i \in [d],\\
-1  & \text{ if } x_i \leq 0,\forall i \in [d],
\end{cases}$ the $2$-layer NTK kernel defined in~\cref{eq:NTK} with $L=2$ satisfy that:
\begin{equation*}
    \mathbb{E}_{(\bm{x}_1,\widetilde{y}_1), (\bm{x}_2,\widetilde{y}_2) \sim \widetilde{P}}\big[K_{\tt NTK}(\bm{x}_1,\bm{x}_2)|\widetilde{y}_1 = \widetilde{y}_2\big]-  \mathbb{E}_{(\bm{x}_1,\widetilde{y}_1), (\bm{x}_2,\widetilde{y}_2) \sim \widetilde{P}}\big[K_{\tt NTK}(\bm{x}_1,\bm{x}_2)|\widetilde{y}_1 \neq \widetilde{y}_2\big]\geq C_{N} = \Theta(1)\,.
\end{equation*}
\end{Proposition}

\begin{proof}
    According to~\citet{bietti2019inductive}, we have for $2$-layer neural network:

    \begin{equation*}
        K_{\tt NTK}(\bm{x}_1,\bm{x}_2) = \left \| \bm{x}_1 \right \|_2 \left \| \bm{x}_2 \right \|_2 \kappa\left(\frac{\left \langle \bm{x}_1, \bm{x}_2 \right \rangle }{\left \| \bm{x}_1 \right \|_2 \left \| \bm{x}_2 \right \|_2}\right)\,,
    \end{equation*}
    where $\kappa (u):= u\kappa_0(u) + \kappa_1(u)$ with $\kappa_0 (u):= \frac{1}{\pi}(\pi - \arccos(u))$ and $\kappa_1 (u):= \frac{1}{\pi}(u(\pi - \arccos(u))+\sqrt{1-u^2})$.
    Denote $\cos(\theta) := \frac{\left \langle \bm{x}_1, \bm{x}_2 \right \rangle }{\left \| \bm{x}_1 \right \|_2 \left \| \bm{x}_2 \right \|_2}$, we have:

    \begin{equation*}
    \begin{split}
        K_{\tt NTK}(\bm{x}_1,\bm{x}_2) & = \left \| \bm{x}_1 \right \|_2 \left \| \bm{x}_2 \right \|_2 \kappa\big(\frac{\left \langle \bm{x}_1, \bm{x}_2 \right \rangle }{\left \| \bm{x}_1 \right \|_1 \left \| \bm{x}_2 \right \|_2}\big)\\
        & = \left \| \bm{x}_1 \right \|_2 \left \| \bm{x}_2 \right \|_2 \kappa\big(\cos(\theta)\big)\\
        & = \left \| \bm{x}_1 \right \|_2 \left \| \bm{x}_2 \right \|_2 \big(\cos(\theta)\kappa_0(\cos(\theta)) + \kappa_1(\cos(\theta))\big)\\
        & = \left \| \bm{x}_1 \right \|_2 \left \| \bm{x}_2 \right \|_2 \big(\cos(\theta)\frac{1}{\pi}(\pi - \theta) + \frac{1}{\pi}(\cos(\theta)(\pi - \theta)+\left | \sin(\theta) \right | )\big)\\
        & = \left \| \bm{x}_1 \right \|_2 \left \| \bm{x}_2 \right \|_2 \big(\cos(\theta)\frac{2}{\pi}(\pi - \theta) + \frac{\left | \sin(\theta) \right |}{\pi}\big)\,.
    \end{split}
    \end{equation*}
    
    According to $y=\begin{cases}
    1  & \text{ if } x_1 > 0 \\
    -1  & \text{ if } x_1 \leq 0
    \end{cases}$, we have if $\bm{x}_3 = -\bm{x}_2$ then $y_3 = -y_2$ and $\cos(\theta') := \frac{\left \langle \bm{x}_1, \bm{x}_3 \right \rangle }{\left \| \bm{x}_1 \right \|_2 \left \| \bm{x}_3 \right \|_2} = -\cos(\theta)$, so $\theta' = \pi - \theta$ and$\left | \sin(\theta') \right | = \left | \sin(\theta) \right |$.

    According to the symmetry of $\widetilde{P}$, we can compute that:

    \begin{equation*}
    \begin{split}
        & \mathbb{E}_{(\bm{x}_1,\widetilde{y}_1), (\bm{x}_2,\widetilde{y}_2) \sim \widetilde{P}}\big[\cos(\theta)\frac{2}{\pi}(\pi - \theta) + \frac{\left | \sin(\theta) \right |}{\pi}|\widetilde{y}_1 = \widetilde{y}_2\big]\\
        = & \mathbb{E}_{(\bm{x}_1,\widetilde{y}_1), (\bm{x}_3,\widetilde{y}_3) \sim \widetilde{P}}\big[\cos(\theta)\frac{2}{\pi}(\pi - \theta) + \frac{\left | \sin(\theta) \right |}{\pi}|\widetilde{y}_1 = \widetilde{y}_3\big]\\ 
        = & \mathbb{E}_{(\bm{x}_1,\widetilde{y}_1), (\bm{x}_3,\widetilde{y}_3) \sim \widetilde{P}}\big[-\cos(\theta')\frac{2}{\pi}\theta' + \frac{\left | \sin(\theta') \right |}{\pi}|\widetilde{y}_1 = \widetilde{y}_3\big]\\    
    \end{split}
    \end{equation*}

    According to the isotropy of $\widetilde{P}$, we have:
    
    \begin{equation*}
    \begin{split}
        & \mathbb{E}_{(\bm{x}_1,\widetilde{y}_1), (\bm{x}_2,\widetilde{y}_2) \sim \widetilde{P}}\big[K_{\tt NTK}(\bm{x}_1,\bm{x}_2)|\widetilde{y}_1 = \widetilde{y}_2\big]-  \mathbb{E}_{(\bm{x}_1,\widetilde{y}_1), (\bm{x}_2,\widetilde{y}_2) \sim \widetilde{P}}\big[K_{\tt NTK}(\bm{x}_1,\bm{x}_2)|\widetilde{y}_1 \neq \widetilde{y}_2\big]\\
        = & \mathbb{E}_{(\bm{x}_1,\widetilde{y}_1) \sim \widetilde{P}}\left \| \bm{x}_1 \right \|_2 \mathbb{E}_{(\bm{x}_2,\widetilde{y}_2) \sim \widetilde{P}}\left \| \bm{x}_2 \right \|_2 \\
        \times & \bigg( \mathbb{E}_{(\bm{x}_1,\widetilde{y}_1), (\bm{x}_2,\widetilde{y}_2) \sim \widetilde{P}}\big[\cos(\theta)\frac{2}{\pi}(\pi - \theta) + \frac{\left | \sin(\theta) \right |}{\pi}|\widetilde{y}_1 = \widetilde{y}_2\big]-  \mathbb{E}_{(\bm{x}_1,\widetilde{y}_1), (\bm{x}_2,\widetilde{y}_2) \sim \widetilde{P}}\big[\cos(\theta)\frac{2}{\pi}(\pi - \theta) + \frac{\left | \sin(\theta) \right |}{\pi}|\widetilde{y}_1 \neq \widetilde{y}_2\big] \bigg)\\
        = & \mathbb{E}_{(\bm{x}_1,\widetilde{y}_1) \sim \widetilde{P}}\left \| \bm{x}_1 \right \|_2 \mathbb{E}_{(\bm{x}_2,\widetilde{y}_2) \sim \widetilde{P}}\left \| \bm{x}_2 \right \|_2 \\
        \times & \bigg( \mathbb{E}_{(\bm{x}_1,\widetilde{y}_1), (\bm{x}_2,\widetilde{y}_2) \sim \widetilde{P}}\big[\cos(\theta)\frac{2}{\pi}(\pi - \theta) + \frac{\left | \sin(\theta) \right |}{\pi}|\widetilde{y}_1 = \widetilde{y}_2\big]-  \mathbb{E}_{(\bm{x}_1,\widetilde{y}_1), (\bm{x}_3,\widetilde{y}_3) \sim \widetilde{P}}\big[-\cos(\theta')\frac{2}{\pi}\theta' + \frac{\left | \sin(\theta') \right |}{\pi}|\widetilde{y}_1 = \widetilde{y}_3\big] \bigg)\\
        = & 2 \mathbb{E}_{(\bm{x}_1,\widetilde{y}_1) \sim \widetilde{P}}\left \| \bm{x}_1 \right \|_2 \mathbb{E}_{(\bm{x}_2,\widetilde{y}_2) \sim \widetilde{P}}\left \| \bm{x}_2 \right \|_2 \mathbb{E}_{(\bm{x}_1,\widetilde{y}_1), (\bm{x}_2,\widetilde{y}_2) \sim \widetilde{P}}\big[\cos(\theta)|\widetilde{y}_1 = \widetilde{y}_2\big]\\
        = & 2 C_{\text{norm}}^2\mathbb{E}_{(\bm{x}_1,\widetilde{y}_1), (\bm{x}_2,\widetilde{y}_2) \sim \widetilde{P}}\big[\cos(\theta)|\widetilde{y}_1 = \widetilde{y}_2\big]\\
        \geq & 2 C_{\text{norm}}^2\mathbb{E}_{(\bm{x}_1,\widetilde{y}_1), (\bm{x}_2,\widetilde{y}_2) \sim \widetilde{P}}\left | \cos(\theta) \right |\,,
    \end{split}
    \end{equation*}
    where the last inequality holds by the fact that the inner product of two data points from the same class is always non-negative in our problem setting. 
    We know that the distribution of the angle between the vectors uniformly distributed on a $d$-dimensional sphere is~\citep{JMLRcai13a}:
    \begin{equation*}
        p(\theta) = \frac{\Gamma(\frac{d}{2})}{\Gamma(\frac{d-1}{2})\sqrt{\pi}}\sin^{d-2} \theta, \quad \theta \in [0 ,\pi]\,.
    \end{equation*}

    Then we use substitution that $x = \cos \theta$, then we have the distribution of the cosine of the angle between the vectors uniformly distributed on a $d$-dimensional sphere is:
    \begin{equation*}
        p(x) = \frac{\Gamma(\frac{d}{2})}{\Gamma(\frac{d-1}{2})\sqrt{\pi}}(1-x^2)^{(d-3)/2}, \quad x \in [-1 ,1]\,.
    \end{equation*}

    Then we have:
    \begin{equation*}
    \begin{split}
        \mathbb{E}_{(\bm{x}_1,\widetilde{y}_1), (\bm{x}_2,\widetilde{y}_2) \sim \widetilde{P}}\left | \cos(\theta) \right | & = \int_{-1}^{1} \left | x \right |p(x) \mathrm{d} x\\
        & = 2\frac{\Gamma(\frac{d}{2})}{\Gamma(\frac{d-1}{2})\sqrt{\pi}}\int_{0}^{1} x(1-x^2)^{(d-3)/2} \mathrm{d} x\\
        & = \frac{2\Gamma(\frac{d}{2})}{\Gamma(\frac{d-1}{2})\sqrt{\pi}(d-1)}\\
        & = \frac{\Gamma(\frac{d}{2})}{\Gamma(\frac{d+1}{2})\sqrt{\pi}}\,.
    \end{split}
    \end{equation*}

    Then, we have:
    \begin{equation*}
        \mathbb{E}_{(\bm{x}_1,\widetilde{y}_1), (\bm{x}_2,\widetilde{y}_2) \sim \widetilde{P}}\big[K_{\tt NTK}(\bm{x}_1,\bm{x}_2)|\widetilde{y}_1 = \widetilde{y}_2\big]-  \mathbb{E}_{(\bm{x}_1,\widetilde{y}_1), (\bm{x}_2,\widetilde{y}_2) \sim \widetilde{P}}\big[K_{\tt NTK}(\bm{x}_1,\bm{x}_2)|\widetilde{y}_1 \neq \widetilde{y}_2\big] \geq \frac{2 C_{\text{norm}}^2\Gamma(\frac{d}{2})}{\Gamma(\frac{d+1}{2})\sqrt{\pi}} = \Theta \left(\frac{1}{\sqrt{d}}\right)\,.
    \end{equation*}
    When we fix the data dimension $d$, it is a constant order.
\end{proof}

\section{Proof for Lemmas in~\cref{sec:proofsketchclass}}
\label{sec:proof_thm1}

\subsection{Proof of~\cref{lemma:lipschitz_concentration}}

Let us restate~\cref{lemma:lipschitz_concentration} as below:

\lipschitzconcentration*

\begin{proof}
    According to the proof of~\citet[Lemma 9]{JMLR_v22_20_974} and~\citet[Lemma 3]{frei2022benign}, we have:

    \begin{equation}
    \begin{split}
        \mathbb{P}_{(\bm{x},y)\sim P}(y \neq \mathrm{sgn}(f(\bm{x};\bm{W}^{(t)}))) & = \mathbb{P}_{(\bm{x},y)\sim P}(y \mathrm{sgn}(f(\bm{x};\bm{W}^{(t)}))<0)\\
        & \leq \eta + \mathbb{P}_{(\bm{x},\widetilde{y})\sim \widetilde{P}}(\widetilde{y}f(\bm{x};\bm{W}^{(t)})<0)\,.
    \end{split}
    \label{eq:first_step}
    \end{equation}

    Denoting the Lipschitz constant of the neural network as $\text{Lip}_{f(\bm{x};\bm{W}^{(t)})}$, since $P_{\text{clust}}$ is $\lambda$-strongly log-concave, then according to~\citet[Theorem 3.16]{wainwright2019high}, for any $t>0$, we have:

    \begin{equation*}
        \mathbb{P} (|\widetilde{y}f(\bm{x};\bm{W}^{(t)}) - \mathbb{E}[\widetilde{y}f(\bm{x};\bm{W}^{(t)})]|\geq t) \leq 2\exp \bigg(-\frac{\lambda}{4} \bigg(\frac{t}{\text{Lip}_{f(\bm{x};\bm{W}^{(t)})}}\bigg)^2\bigg)\,.
    \end{equation*}

    Choosing $t := \mathbb{E}[\widetilde{y}f(\bm{x};\bm{W}^{(t)})]$, we have:

    \begin{equation*}
        \mathbb{P} (|\widetilde{y}f(\bm{x};\bm{W}^{(t)}) - \mathbb{E}[\widetilde{y}f(\bm{x};\bm{W}^{(t)})]|\geq  \mathbb{E}[\widetilde{y}f(\bm{x};\bm{W}^{(t)})]) \leq 2\exp \bigg(-\frac{\lambda}{4} \bigg(\frac{ \mathbb{E}[\widetilde{y}f(\bm{x};\bm{W}^{(t)})]}{\text{Lip}_{f(\bm{x};\bm{W}^{(t)})}}\bigg)^2\bigg)\,,
    \end{equation*}

    which implies:

    \begin{equation*}
    \begin{split}
        \mathbb{P}_{(\bm{x},\widetilde{y})\sim \widetilde{P}}(\widetilde{y}f(\bm{x};\bm{W}^{(t)})<0) & = \mathbb{P}(\widetilde{y}f(\bm{x};\bm{W}^{(t)}) - \mathbb{E}[\widetilde{y}f(\bm{x};\bm{W}^{(t)})] < -\mathbb{E}[\widetilde{y}f(\bm{x};\bm{W}^{(t)})])\\
        & \leq \exp \bigg(-\frac{\lambda}{4} \bigg(\frac{ \mathbb{E}[\widetilde{y}f(\bm{x};\bm{W}^{(t)})]}{\text{Lip}_{f(\bm{x};\bm{W}^{(t)})}}\bigg)^2\bigg)\,.
    \end{split}
    \end{equation*}
    Incorporating it into~\cref{eq:first_step}, we conclude the proof.

\end{proof}

\subsection{Proof of~\cref{lemma:almost_linear}}

Let us restate~\cref{lemma:almost_linear} as below:

\almostlinear*

\begin{proof}
    We can directly calculate that: 

    \begin{equation*}
    \begin{split}
        & f(\bm{x}; \bm{W}) - f(\bm{x}; \bm{W}') - \left \langle \nabla f(\bm{x}; \bm{W}' ), \bm{W} - \bm{W}'\right \rangle\\
        = & \bm{W}_L(\bm{h}_{i,L-1}-\bm{h}_{i,L-1}')-\sum_{l=1}^{L-1} \bm{W}_L' \bigg( \prod_{r = l+1}^{L-1} (\bm{D}_{i,r}'\bm{W}_r')\bigg) \bm{D}_{i,l}'(\bm{W}_l-\bm{W}_l')\bm{h}_{i,l-1}'\,.
    \end{split}
    \end{equation*}

    By~\citet[Claim 11.2]{allen2019convergence}, when $\bm{W}, \bm{W}' \in \mathcal{B} (\bm{W}^{(0)},\omega )$, there exist diagonal matrices $\bm{D}_{i,l}''\in \mathbb{R}^{m\times m}$ with entries in $\{ +1, -1\}$ such that $\left \| \bm{D}_{i,l}'' \right \|_0 \leq \mathcal{O}(m \omega^{2/3}L)$ and:
    \begin{equation*}
        \bm{h}_{i,L-1}-\bm{h}_{i,L-1}' = \sum_{l=1}^{L-1} \bigg( \prod_{r = l+1}^{L-1} (\bm{D}_{i,r} + \bm{D}_{i,r}'')\bm{W}_r\bigg)(\bm{D}_{i,l}+\bm{D}_{i,l}'')(\bm{W}_l-\bm{W}_l')\bm{h}_{i,l-1}',\quad \forall i \in [n]\,.
    \end{equation*}
    
    Therefore, we have:

    \begin{equation}
    \begin{split}
        & f(\bm{x}; \bm{W}) - f(\bm{x}; \bm{W}') - \left \langle \nabla f(\bm{x}; \bm{W}' ), \bm{W} - \bm{W}'\right \rangle \\
        = & \sum_{l=1}^{L-1} \bm{W}_L\bigg( \prod_{r = l+1}^{L-1} (\bm{D}_{i,r} + \bm{D}_{i,r}'')\bm{W}_r\bigg)(\bm{D}_{i,l}+\bm{D}_{i,l}'')(\bm{W}_l-\bm{W}_l')\bm{h}_{i,l-1}'\\
        & -\sum_{l=1}^{L-1} \bm{W}_L' \bigg( \prod_{r = l+1}^{L-1} (\bm{D}_{i,r}'\bm{W}_r')\bigg) \bm{D}_{i,l}'(\bm{W}_l-\bm{W}_l')\bm{h}_{i,l-1}'\,.
    \end{split}
    \label{eq:linear_bound_1}
    \end{equation}

    According to~\citet[Lemma B.1]{cao2019generalization}, for $\omega \leq \mathcal{O}(L^{-9/2}(\log m)^{-3})$, then with probability at least $1 - \mathcal{O}(nL^2)\exp(-\Omega(m\omega^{2/3}L))$, we have:
    
    \begin{equation}
        \frac{1}{2} \left \| \bm{x}_i \right \|_2\leq \left \| \bm{h}_{i,l-1}' \right \|_2 \leq \frac{3}{2} \left \| \bm{x}_i \right \|_2 \leq \frac{3}{2} C_{\text{norm}} = \Theta(1),\forall i \in [n], l \in [L-1]\,.
    \label{eq:feature_norm}
    \end{equation}

    According to~\citet[Lemma 8.7]{allen2019convergence}, for $s \in \big[\Omega\big(\frac{1}{\log m}\big), \mathcal{O}\big(\frac{m}{L^3 \log m}\big)\big]$, $\omega = \mathcal{O}(L^{-3/2})$, with probability at least $1-\exp(-\Omega(s\log m))$, we have:

    \begin{equation*}
        \left \| \bm{W}_L\bigg( \prod_{r = l+1}^{L-1} (\bm{D}_{i,r} + \bm{D}_{i,r}'')\bm{W}_r\bigg)(\bm{D}_{i,l}+\bm{D}_{i,l}'') - \bm{W}_L' \bigg( \prod_{r = l+1}^{L-1} (\bm{D}_{i,r}'\bm{W}_r')\bigg) \bm{D}_{i,l}'\right \|_2 \leq \mathcal{O}(\sqrt{L^3 s \log m + \omega^2 L^3 m})\,.
    \end{equation*}

    Then, taking $s := \Theta(m\omega^2)$, we have $m\omega^2 \leq \mathcal{O}\big(\frac{m}{L^3 \log m}\big)$, that is $\omega \leq \mathcal{O}(L^{-3/2}(\log m)^{-1/2})$, with probability at least $1-\exp(-\Omega(m\omega^2\log m))$, we have:

    \begin{equation}
        \left \| \bm{W}_L\bigg( \prod_{r = l+1}^{L-1} (\bm{D}_{i,r} + \bm{D}_{i,r}'')\bm{W}_r\bigg)(\bm{D}_{i,l}+\bm{D}_{i,l}'') - \bm{W}_L' \bigg( \prod_{r = l+1}^{L-1} (\bm{D}_{i,r}'\bm{W}_r')\bigg) \bm{D}_{i,l}'\right \|_2 \leq \mathcal{O}(\sqrt{\omega^2 L^3 m \log m})\,.
    \label{eq:gradient_lazy_norm_bound}
    \end{equation}

    Take~\cref{eq:feature_norm} and~\cref{eq:gradient_lazy_norm_bound} into~\cref{eq:linear_bound_1}, for $\omega \leq \mathcal{O}(L^{-9/2}(\log m)^{-3})$, then with probability at least $1 - \exp(-\Omega(m\omega^2\log m)) - \mathcal{O}(nL^2)\exp(-\Omega(m\omega^{2/3}L))$, we have:
    
    \begin{equation*}
    \begin{split}
        |f(\bm{x}; \bm{W}) - f(\bm{x}; \bm{W}') - \left \langle \nabla f(\bm{x}; \bm{W}' ), \bm{W} - \bm{W}'\right \rangle| \leq  \mathcal{O}(\sqrt{\omega^2 L^3 m \log m})\sum_{l=1}^{L-1}\left \| \bm{W}_l-\bm{W}_l' \right \|_2\,,
    \end{split}
    \end{equation*}
    which concludes the proof.
\end{proof}

\subsection{Proof of~\cref{lemma:change_every_step}}

Let us restate~\cref{lemma:change_every_step}:

\changeeverystep*

\begin{proof}
    By~\cref{lemma:almost_linear}, for $\omega = \mathcal{O} (L^{-9/2}(\log m)^{-3} )$, with probability at least $1- \exp(-\Omega(m\omega^2\log m))-\mathcal{O}(nL^2)\exp(-\Omega(m\omega^{2/3}L))$, we have:
    \begin{equation*}
        \left|f(\bm{x}; \bm{W}^{(t+1)}) - f(\bm{x}; \bm{W}^{(t)}) - \left \langle \nabla f(\bm{x}; \bm{W}^{(t)} ), \bm{W}^{(t+1)} - \bm{W}^{(t)}\right \rangle \right| \leq \mathcal{O}(\sqrt{\omega^2 L^3 m \log m})\sum_{l=1}^{L-1}\left \| \bm{W}^{(t+1)}_l - \bm{W}^{(t)}_{l} \right \|_2 \,.
    \end{equation*}

    Since $\widetilde{y} \in \left \{ \pm 1 \right \}$, we can calculate that:

    \begin{equation*}
    \begin{split}
        & \widetilde{y}[f(\bm{x}; \bm{W}^{(t+1)}) - f(\bm{x}; \bm{W}^{(t)})] \\
        \geq & \widetilde{y}\left \langle \nabla f(\bm{x}; \bm{W}^{(t)} ), \bm{W}^{(t+1)} - \bm{W}^{(t)}\right \rangle - \mathcal{O}(\sqrt{\omega^2 L^3 m \log m})\sum_{l=1}^{L-1}\left \| \bm{W}^{(t+1)}_l - \bm{W}^{(t)}_{l} \right \|_2 \\
        \overset{(a)} = & \widetilde{y} \alpha g_i^{(t)}
        \left \langle \nabla f(\bm{x}; \bm{W}^{(t)} ), y_i \nabla f(\bm{x}_i; \bm{W}^{(t)} ) \right \rangle - \mathcal{O}(\sqrt{\omega^2 L^3 m \log m})\sum_{l=1}^{L-1}\left \| \bm{W}^{(t+1)}_l - \bm{W}^{(t)}_{l} \right \|_2\\
        = & \alpha g_i^{(t)} \left \langle \widetilde{y} \nabla f(\bm{x}; \bm{W}^{(t)} ), y_i \nabla f(\bm{x}_i; \bm{W}^{(t)} ) \right \rangle - \mathcal{O}(\sqrt{\omega^2 L^3 m \log m})\sum_{l=1}^{L-1}\left \| \bm{W}^{(t+1)}_l - \bm{W}^{(t)}_{l} \right \|_2\,,
    \end{split}        
    \end{equation*}
    where $(\bm {x}_i, y_i )$ is the random selected training sample at step $t + 1$, and $(a)$ uses~\cref{alg:algorithm_SGD} and definition of $g_i^{(t)}$ that:
    \begin{equation*}
    \begin{split}
        \bm{W}^{(t+1)} - \bm{W}^{(t)} & = - \alpha\cdot \nabla \ell \big( y_i f(\bm x_i; \bm{W}^{(t)})  \big)\\
        & = - \alpha \ell' \big( y_i f(\bm x_i; \bm{W}^{(t)})  \big)\cdot \nabla\big( y_i f(\bm x_i; \bm{W}^{(t)})  \big) \\
        & = \alpha g_i^{(t)} y_i \nabla f(\bm x_i; \bm{W}^{(t)}) \,.
    \end{split}
    \end{equation*}
    Finally we conclude the proof.
\end{proof}

\subsection{Proof of~\cref{lemma:normalized_margin}}

Let us restate~\cref{lemma:normalized_margin} as below:

\normalizedmarginbound*

\begin{proof}
    According to the~\cref{lemma:change_every_step}, $\forall t \geq 0$, for $\omega = \mathcal{O} (L^{-9/2}(\log m)^{-3} )$, with probability at least $1- \exp(-\Omega(m\omega^2\log m))-\mathcal{O}(nL^2)\exp(-\Omega(m\omega^{2/3}L))$, we have:
    \begin{equation}
    \begin{split}
        & \mathbb{E}_{(\bm{x},\widetilde{y})\sim \widetilde{P}}[\widetilde{y}(f(\bm{x};\bm{W}^{(t+1)}) - f(\bm{x};\bm{W}^{(t)}))]\\
        \geq & \alpha g_i^{(t)} \mathbb{E}_{(\bm{x},\widetilde{y})\sim \widetilde{P}} \left \langle \widetilde{y} \nabla f(\bm{x}; \bm{W}^{(t)} ), y_i \nabla f(\bm{x}_i; \bm{W}^{(t)} ) \right \rangle - \mathcal{O}(\sqrt{\omega^2 L^3 m \log m})\sum_{l=1}^{L-1}\left \| \bm{W}^{(t+1)}_l - \bm{W}^{(t)}_{l} \right \|_2\\
        \geq & \alpha g_i^{(t)} \mathbb{E}_{(\bm{x},\widetilde{y})\sim \widetilde{P}} \left \langle \widetilde{y} \nabla f(\bm{x}; \bm{W}^{(t)} ), y_i \nabla f(\bm{x}_i; \bm{W}^{(t)} ) \right \rangle - \mathcal{O}(\sqrt{\omega^2 L^5 \log m})\\
        := & \alpha g_i^{(t)} \mathbb{E}_{(\bm{x},\widetilde{y})\sim \widetilde{P}} \left \langle \widetilde{y} \nabla f(\bm{x}; \bm{W}^{(t)} ), y_i \nabla f(\bm{x}_i; \bm{W}^{(t)} ) \right \rangle - \varepsilon \,,
    \end{split}
    \label{eq:expectation_yf_diff_t}
    \end{equation}

    where the second inequality use the result of lazy-training that $\bm{W}^{(t+1)} \in \mathcal{B} (\bm{W}^{(t)},\frac{1}{\sqrt{m}} )$ and $(\bm {x}_i, y_i )$ is the random selected training sample at step $t + 1$.

    By the definition of $\varepsilon$ we have: 
    \begin{equation*}
        \varepsilon = \mathcal{O}(\sqrt{\omega^2 L^3 \log m}) = \mathcal{O}(L^{-2}(\log m)^{-5/2})\,.
    \end{equation*}
    
    According to~\cref{assumption:kernel}, we have:
    \begin{equation}
    \begin{split}
        & \mathbb{E}_{(\bm{x},\widetilde{y})\sim \widetilde{P}} \left \langle \widetilde{y} \nabla f(\bm{x}; \bm{W}^{(t)} ), y_i \nabla f(\bm{x}_i; \bm{W}^{(t)} ) \right \rangle\\
        = & \mathbb{P}(i \in \mathcal{C}) \mathbb{E}_{(\bm{x},\widetilde{y})\sim \widetilde{P}} \bigg[\left \langle \widetilde{y} \nabla f(\bm{x}; \bm{W}^{(t)} ), y_i \nabla f(\bm{x}_i; \bm{W}^{(t)} ) \right \rangle|i \in \mathcal{C}\bigg] \\
        & + \mathbb{P}(i \in \mathcal{C}') \mathbb{E}_{(\bm{x},\widetilde{y})\sim \widetilde{P}} \bigg[\left \langle \widetilde{y} \nabla f(\bm{x}; \bm{W}^{(t)} ), y_i \nabla f(\bm{x}_i; \bm{W}^{(t)} ) \right \rangle|i \in \mathcal{C}' \bigg]\\
        = & \mathbb{P}(i \in \mathcal{C}) \mathbb{P}(\widetilde{y} = y_i)\mathbb{E}_{(\bm{x},\widetilde{y})\sim \widetilde{P}} \bigg[\left \langle \nabla f(\bm{x}; \bm{W}^{(t)} ), \nabla f(\bm{x}_i; \bm{W}^{(t)} ) \right \rangle|i \in \mathcal{C}, \widetilde{y} = y_i\bigg]\\
        & - \mathbb{P}(i \in \mathcal{C}) \mathbb{P}(\widetilde{y} \neq y_i)\mathbb{E}_{(\bm{x},\widetilde{y})\sim \widetilde{P}} \bigg[\left \langle \nabla f(\bm{x}; \bm{W}^{(t)} ), \nabla f(\bm{x}_i; \bm{W}^{(t)} ) \right \rangle|i \in \mathcal{C}, \widetilde{y} \neq y_i\bigg]\\
        & + \mathbb{P}(i \in \mathcal{C}') \mathbb{P}(\widetilde{y} = y_i)\mathbb{E}_{(\bm{x},\widetilde{y})\sim \widetilde{P}} \bigg[\left \langle \nabla f(\bm{x}; \bm{W}^{(t)} ), \nabla f(\bm{x}_i; \bm{W}^{(t)} ) \right \rangle|i \in \mathcal{C}', \widetilde{y} = y_i\bigg]\\
        & - \mathbb{P}(i \in \mathcal{C}') \mathbb{P}(\widetilde{y} \neq y_i)\mathbb{E}_{(\bm{x},\widetilde{y})\sim \widetilde{P}} \bigg[\left \langle \nabla f(\bm{x}; \bm{W}^{(t)} ), \nabla f(\bm{x}_i; \bm{W}^{(t)} ) \right \rangle|i \in \mathcal{C}', \widetilde{y} \neq y_i\bigg]\\
        \geq & (1 - \eta) \times \frac{1}{2} \times C_{N} + [\eta \times \frac{1}{2} \times (-C_{N}) ]\\
        = & \frac{1-2\eta}{2} C_{N}\,,
    \end{split}
    \label{eq:expectation_yf_t}
    \end{equation}

    where the $\mathcal{C}$ and $\mathcal{C}'$ represent the subset of training data for clean labels and noisy labels respectively. And $(\bm {x}_i, y_i )$ is the random selected training sample at step $t + 1$.

    According to~\citet[Lemma B.1]{cao2019generalization}. For $\omega \leq \mathcal{O}(L^{-9/2}(\log m)^{-3})$, then with probability at least $1 - \mathcal{O}(nL^2)\exp(-\Omega(m\omega^{2/3}L))$, $\frac{1}{2} \leq \left | f(\bm {x}_i;\bm{W}^{(t)})) \right | \leq \frac{3}{2} \left \| \bm{x}_i \right \|_2 \leq \frac{3}{2} C_{\text{norm}},\forall i \in [n]$.

    That means:
    \begin{equation}
        g_i^{(t)} = -\ell'(y_i f(\bm{x}_i; \bm{W}^{(t)} )) = \frac{1}{1+\exp\big(y_i f(\bm{x}_i; \bm{W}^{(t)}\big)} \geq \frac{1}{1+\exp(\frac{3}{2}C_{\text{norm}})}\,, \quad \forall i \in [n]\,.
    \label{eq:bound_git}
    \end{equation}
    
    Take~\cref{eq:expectation_yf_t} and~\cref{eq:bound_git} into~\cref{eq:expectation_yf_diff_t}, we have for $\omega \leq \mathcal{O}(L^{-9/2}(\log m)^{-3})$, then with probability at least $1- \exp(-\Omega(m\omega^2\log m)) - \mathcal{O}(nL^2)\exp(-\Omega(m\omega^{2/3}L))$:
    \begin{equation*}
    \begin{split}
        & \mathbb{E}_{(\bm{x},\widetilde{y})\sim \widetilde{P}}[\widetilde{y}(f(\bm{x};\bm{W}^{(t+1)}) - f(\bm{x};\bm{W}^{(t)}))]\\
        \geq & \alpha g_i^{(t)} \mathbb{E}_{(\bm{x},\widetilde{y})\sim \widetilde{P}} \left \langle \widetilde{y} \nabla f(\bm{x}; \bm{W}^{(t)} ), y_i \nabla f(\bm{x}_i; \bm{W}^{(t)} ) \right \rangle - \varepsilon \\
         \geq & \alpha g_i^{(t)} \frac{1-2\eta}{2} C_{N} - \varepsilon \\
         \geq & \frac{\alpha}{1+\exp(\frac{3}{2}C_{\text{norm}})} \frac{1-2\eta}{2} C_{N} - \varepsilon\\
         = &\Theta\big(\alpha(1-2\eta)C_{N}\big)\,,
    \end{split}
    \end{equation*}

    where the last inequality hold for $\alpha  \gtrsim \varepsilon = \mathcal{O}(L^{-2}(\log m)^{-5/2})$.
    
    Since the symmetry of the random Gaussian initialization and zero-mean of $\tilde{y}$, we have $\mathbb{E}_{(\bm{x},\widetilde{y})\sim \widetilde{P}}[\widetilde{y}f(\bm{x};\bm{W}^{(0)})] = 0$, then we have:

    \begin{equation*}
    \begin{split}
        & \mathbb{E}_{(\bm{x},\widetilde{y})\sim \widetilde{P}}[\widetilde{y}f(\bm{x};\bm{W}^{(t)})]\\
        = & \mathbb{E}_{(\bm{x},\widetilde{y})\sim \widetilde{P}}[\widetilde{y}f(\bm{x};\bm{W}^{(0)})] + \sum_{s=0}^{n-1}\mathbb{E}_{(\bm{x},\widetilde{y})\sim \widetilde{P}}[\widetilde{y}(f(\bm{x};\bm{W}^{(s+1)}) - f(\bm{x};\bm{W}^{(s)}))]\\
        \geq & 0 + \sum_{t=0}^{t-1} \Theta\big(\alpha(1-2\eta)C_{N}\big)\\
        = & \Theta\big(t\alpha(1-2\eta)C_{N}\big)\,.
    \end{split}
    \end{equation*}
\end{proof}

\section{Optimization result for the classification problem}
\label{sec:optimization}
Before presenting the optimization result, we first introduce the following assumption:

\begin{assumption}
\label{assumption:assumption_optimization}
    For $(\bm{x}_1,\widetilde{y}_1), (\bm{x}_2,\widetilde{y}_2) \sim \widetilde{P}$, if $y_1 \neq y_2$, then $\left \| \bm{x}_1 - \bm{x}_2 \right \|_2 \geq \phi$ for some $\phi > 0$.
\end{assumption}

Then we present the theorem and its proof.
\begin{theorem}
\label{thm:optimization}
    Given a DNN defined by~\cref{eq:deep_network} and trained by~\cref{alg:algorithm_SGD} with the training data satisfy~\cref{assumption:distribution_1} and~\cref{assumption:assumption_optimization}, then for the step size $\alpha =\mathcal{O}(n^{-3}L^{-9}m^{-1})$, the width $m = \widetilde{\Omega}(\text{poly}(n, \phi^{-1},L))\Omega(1/\delta)$ and the maximum number of iteration $t = \mathcal{O}\big(\text{poly}(n, \phi^{-1},L)\big)\mathcal{O}(1/\delta)$ then with high probability,~\cref{alg:algorithm_SGD} can find a point $\bm{W}^{(t)}$ such that $\hat{L}(\bm{W}^{(t)}) \leq \delta$.
\end{theorem}

\begin{proof}
    Recall our loss function is:
    
    \begin{equation*}
        \ell (z) = \log(1+\exp(-z))\,.
    \end{equation*}
    
    And we can derive that:
    
    \begin{equation*}
        \ell'(z) = -\frac{1}{1+e^{z}}, \quad \ell''(z) = \frac{e^{z}}{(1+e^{z})^2}\,.
    \end{equation*}

    It is easy to verify that:

    \begin{equation*}
        \ell'(z) < 0, \quad \lim_{z \to \infty}\ell(z) = 0, \quad \lim_{z \to \infty}\ell'(z) = 0, \quad -\ell'(z) \geq \min \left \{ \frac{1}{2}, \frac{1}{2}\ell(z) \right \}, \quad \left | \ell''(z) \right |\leq \frac{1}{4}\,.  
    \end{equation*}

    Then according to~\citet[Theorem 4.1]{zou2020gradient}, we conclude the proof.
\end{proof}

\section{Proof of~\cref{thm:min_eigen_NTK}}
\label{sec:proofthm2}
Before proving~\cref{thm:min_eigen_NTK}, we first introduce the following lemmas.

\begin{lemma}[Minimum eigenvalue of sample covariance matrix. Adapted from Lemma 1 in~\citet{NEURIPS2021_f7541864}]
\label{lemma:min_eigen_XXT}
Let $\bm{X} = [\bm{x}_1, \cdots, \bm{x}_n]^{\top} \in \mathbb{R}^{n\times d} $ be a matrix with i.i.d. columns that satisfy~Assumptions~\ref{assumption:distribution_sphere}, and let $\hat{\bm{\Sigma}} = \bm{X}^{\top}\bm{X}/n $, and $\bm{\Sigma} = \mathbb{E}[\bm{x}\bm{x}^{\top}]$. Then, for every $s \geq 0$, with probability at least $1-2e^{-n}$, we have:
\begin{equation*}
    \lambda_{\min}(\hat{\bm{\Sigma}}) \geq \lambda_{\min}(\bm{\Sigma})\bigg(\frac{3}{4}-\frac{c}{4}\sqrt{\frac{d}{n}}\bigg)^2, \quad \text{if}\quad  n \geq d\,,
\end{equation*}
and 
\begin{equation*}
    \lambda_{\min}(\hat{\bm{\Sigma}}) \geq \lambda_{\min}(\bm{\Sigma})\bigg(\sqrt{\frac{d}{n}}-\frac{c+6}{4}\bigg)^2, \quad \text{if}\quad  n < d\,,
\end{equation*}
where we have an absolute constant $c = 2^{3.5}\sqrt{\log(9)}$.
\end{lemma}

\begin{proof}
    According to~\cref{assumption:distribution_sphere}, we have $ \max_i \left \| \bm{x}_i \right \|_{\psi_2}$ is bounded and $\left \| \bm{x}_i \right \|_{\bm{\Sigma}^{\dagger }} = \sqrt{d}$ almost sure for all $i \in [n]$. Without loss of generality, take $\max_i \left \| \bm{x}_i \right \|_{\psi_2} \leq \frac{1}{2}$ into~\citet[Lemma 1]{NEURIPS2021_f7541864}, then choosing $s := n$,which conclude the proof.
\end{proof}

Then we are ready to prove~\cref{thm:min_eigen_NTK}.

\begin{proof}[Proof of~\cref{thm:min_eigen_NTK}]
According to~\cref{lemma:NTK_matrix_recursive_form} and~\citet[Theorem 1]{zhu2022generalization}, we have:
\begin{equation*}
    \lambda_0 \geq 2\mu_{1}^{2}\lambda_{\min}(\bm{XX}^{\top})\,,
\end{equation*}
where $\mu_1$ is the $1$-st Hermite coefficient of the ReLU activation function.

We know that, $\bm{XX}^{\top}$ and $\bm{X}^{\top}\bm{X}$ have the same non-zero eigenvalues. Then according to~\cref{lemma:min_eigen_XXT} and~\cref{assumption:distribution_sphere},  with probability at least $1-2e^{-n}$, we have:

{\bf Case 1:} $n \geq d$
\begin{equation*}
\begin{split}
    \lambda_0 & \geq 2\mu_{1}^{2}\lambda_{\min}(\bm{XX}^{\top})\\
    & = 2\mu_{1}^{2} n\lambda_{\min}(\hat{\bm{\Sigma}}) \\
    & \geq 2\mu_{1}^{2}\frac{n}{d}\bigg(\frac{3}{4}-\frac{c}{4}\sqrt{\frac{d}{n}}\bigg)^2,\quad \text{if}\quad  n \geq d\,.
\end{split}
\end{equation*}
{\bf Case 2:} $n < d$
\begin{equation*}
\begin{split}
    \lambda_0 & \geq 2\mu_{1}^{2}\lambda_{\min}(\bm{XX}^{\top})\\
    & = 2\mu_{1}^{2}n\lambda_{\min}(\hat{\bm{\Sigma}}) \\
    & \geq 2\mu_{1}^{2}\frac{n}{d}\bigg(\sqrt{\frac{d}{n}}-\frac{c+6}{4}\bigg)^2, \quad \text{if}\quad  n < d\,.
\end{split}
\end{equation*}
where we have an absolute constant $c = 2^{3.5}\sqrt{\log(9)} \approx 16.77$.
\end{proof}

\section{Supplementary proofs for~\cref{thm:kernel}}
\label{sec:proofsec4}
In this section, we present the proofs of \cref{thm:kernel} in \cref{sec:kernel}.

\subsection{A Precise Form of the~\cref{thm:kernel}}
\label{ssec:precise_form_main_thm}

\begin{theorem}[Precise form of~\cref{thm:kernel}]
\label{thm:precise_form}

Let $\alpha$, $\beta$ and $\gamma$ be three non-negative parameters depends on the Laplace kernel, under~\cref{assumption:distribution_high_d},~\ref{assumption:distribution_noise} and~\ref{assumption:distribution_sphere}, let $0 < \delta < \frac{1}{2}$, $0 < \theta \leq 1/2$, the ground-truth function $f_{\rho}$ lies in a RKHS by~\cref{assumption:distribution_5} and $d$ large enough, suppose that, $\omega \leq poly(1/n, \lambda_0, 1/L, 1/\log(m), \epsilon, 1/\log(1/\delta'), \kappa)$, $m\geq poly(1/\omega)$ and $\kappa = \mathcal{O}(\frac{\epsilon }{\log(n/\delta')})$. then for any given $\varepsilon > 0$, it holds with probability at least $1-2\delta-\delta'-d^{-2}-2e^{-n}$.

\begin{equation*}
\begin{split}
    \mathbb{E} \left \| f_{\tt nn} -f_{\rho} \right \|_{L_{\rho_X}^2}^{2} \leq & n^{-\theta}\log^4 (\frac{2}{\delta})  + \frac{\sigma_{\epsilon}^2\beta}{d}\mathcal{N}_{\widetilde{\bm{X}}}^{\gamma} + \frac{\sigma_{\epsilon}^2 \log ^{2+4\varepsilon}d}{\gamma^2d^{4\theta-1}}  + \mathcal{O} \bigg(\big(\epsilon + \frac{\sqrt{n}}{\lambda_0}\omega^{1/3}L^{5/2}\sqrt{m \log m}+\frac{n^{3/2}}{\lambda_0^3\kappa}\big)^2\bigg)\,,
\end{split}
\end{equation*}

where:

\begin{equation*}
\lambda_0 \geq \begin{cases}
  & 2\mu_{1}^{2}\frac{n}{d}\bigg(\frac{3}{4}-\frac{c}{4}\sqrt{\frac{d}{n}}\bigg)^2, \quad \text{if} \quad n \geq d, \\
  & 2\mu_{1}^{2}\frac{n}{d}\bigg(\sqrt{\frac{d}{n}}-\frac{c+6}{4}\bigg)^2, \quad \text{if} \quad n < d\,.
\end{cases}
\end{equation*}

where we have an absolute constant $c = 2^{3.5}\sqrt{\log(9)}$, $\widetilde{\bm{X}} := \beta \bm{X}\bm{X}^{\top}/d + \alpha \bm{1}\bm{1}^{\top}, \mathcal{N}_{\widetilde{\bm{X}}}^{\gamma} = \sum_{i=0}^{n-1} \frac{\lambda_i(\widetilde{\bm{X}})}{(\lambda_i(\widetilde{\bm{X}} +\gamma ))^2}, \left \| f \right \|_{L_{\rho_X}^2}^{2}=\int_{X}\left | f(\bm{x}) \right |^2d\rho_{X}(\bm{x})$, and $\mu_1$ is the $1$-st Hermite coefficient of the ReLU activation function.

\end{theorem}
{\bf Remark:} The three non-negative parameters $\alpha$, $\beta$, and $\gamma$ depend on the linearization of the Laplace kernel in the high dimension setting, refer to~\citep{liu2021kernel} for details.

\subsection{Propositions}
\label{ssec:Propositions}

We present several propositions that are needed for our \cref{thm:kernel} as below.

\begin{Proposition}[Convergence to the NTK at initialization. Adapted from Theorem 3.1 in~\citet{arora2019exact}]
\label{prop:Convergence_to_the_NTK_at_initialization}
Fix $\epsilon > 0$ and $\delta \in (0, 1)$. Suppose that $m \geq \Omega(\frac{L^6}{\epsilon^4}\log(L/\delta))$. Then for any inputs $\left \| \bm{x}_1 \right \| \leq 1$, $\left \| \bm{x}_2 \right \| \leq 1$, with probability at least $1-\delta$ we have:
\begin{equation*}
    \left | \left \langle \frac{\partial f(\bm{x}_1; \bm{W})}{\partial \bm{W}} , \frac{\partial f(\bm{x}_2;\bm{W})}{\partial \bm{W}} \right \rangle - K_{\tt NTK}(\bm{x}_1, \bm{x}_2) \right | = \mathcal{O} (\epsilon L) \,.
\end{equation*}
\end{Proposition}

\begin{Proposition}[Equivalence between trained neural network and kernel regression]
\label{prop:equivalence_ntk_network}
    Suppose that, $\omega \leq poly(1/n, \lambda_0, 1/L, 1/\log(m), \epsilon, 1/\log(1/\delta), \kappa)$, $m\geq poly(1/\omega)$, $\bm{x}_{te}$ satisfy~\cref{assumption:distribution_sphere} and $\kappa = \mathcal{O}(\frac{\epsilon }{\log(n/\delta)})$. Then w.p. at least $1-\delta$ over random initialization, we have:

    \begin{equation*}
        \left | f_{\tt nn}(\bm{x}_{te}) - f_{\tt NTK}(\bm{x}_{te}) \right | \leq \mathcal{O} \big(\epsilon + \frac{\sqrt{n}}{\lambda_0}\omega^{1/3}L^{5/2}\sqrt{m \log m}+\frac{n^{3/2}}{\lambda_0^3\kappa}\big)\,.
    \end{equation*}
\end{Proposition}

\begin{Proposition}[Adapted from Theorem 2 in~\citet{liu2021kernel}]
\label{prop:diff_kernel_ground_truth}

Let $\alpha$, $\beta$ and $\gamma$ be three non-negative parameters depends on the laplace kernel, under~\cref{assumption:distribution_high_d},~\ref{assumption:distribution_noise} and~\ref{assumption:distribution_sphere} let $0 < \delta < \frac{1}{2}$, $\theta = \frac{1}{2}-\frac{2}{8+m}$, ground-truth function $f_{\rho}$ lies in a RKHS and $d$ large enough, then for any given $\varepsilon > 0$, it holds with probability at least $1-2\delta -d^{-2}$.

\begin{equation*}
    \mathbb{E} \left \| f_{\tt NTK} -f_{\rho} \right \|_{L_{\rho_X}^2}^{2} \leq n^{-2\theta r}\log^4 (\frac{2}{\delta})  + \frac{\sigma_{\epsilon}^2\beta}{d}\mathcal{N}_{\widetilde{\bm{X}}}^{\gamma} + \frac{\sigma_{\epsilon}^2 \log ^{2+4\varepsilon}d}{\gamma^2d^{4\theta-1}}\,,
\end{equation*}

where $\widetilde{\bm{X}} := \beta \bm{X}\bm{X}^{\top}/d + \alpha \bm{1}\bm{1}^{\top}, \mathcal{N}_{\widetilde{\bm{X}}}^{\gamma} = \sum_{i=0}^{n-1} \frac{\lambda_i(\widetilde{\bm{X}})}{(\lambda_i(\widetilde{\bm{X}} +\gamma ))^2}, \left \| f \right \|_{L_{\rho_X}^2}^{2}=\int_{X}\left | f(\bm{x}) \right |^2d\rho_{X}(\bm{x})$.
\end{Proposition}

\subsection{Proof of~\cref{prop:equivalence_ntk_network}}
\label{sec:proof_of_equivalence_ntk_network}

Before proving~\cref{prop:equivalence_ntk_network}, we need the following lemmas:

\begin{lemma}[Gradient Perturbation $\to$ Kernel Perturbation]
\label{lemma:Gradient_Perturbation_to_Kernel_Perturbation}
For any two data point $\bm x_1$, $\bm x_2$ that satisfy~\cref{assumption:distribution_sphere} and the neural network defined in \cref{eq:deep_network}, 
if $\left \|\frac{\partial f( \bm{x}_1; \bm{W}^{(t)})}{\partial \bm{W}} - \frac{\partial f( \bm{x}_1; \bm{W}^{(0)})}{\partial \bm{W}}\right \| \leq \epsilon$ and $\left \|\frac{\partial f( \bm{x}_2; \bm{W}^{(t)})}{\partial \bm{W}} - \frac{\partial f( \bm{x}_2; \bm{W}^{(0)})}{\partial \bm{W}}\right \| \leq \epsilon$, we have
\begin{equation*}
    \left | K_{\tt NTK}^{(t)}(\bm{x}_1, \bm{x}_2) - K_{\tt NTK}^{(0)}(\bm{x}_1, \bm{x}_2)    \right |  = \mathcal{O}(\epsilon)\,,
\end{equation*}

where the $\bm{K}_{\tt NTK}^{(t)}$ means the NTK kernel defined in~\cref{eq:NTK} for neural networks~\cref{eq:deep_network} at training time $t$.

\end{lemma}

\begin{proof}
According to Lemma 5 in~\citet{zhu2022generalization} and our network~\cref{eq:deep_network}, we have:
\begin{equation*}
    \frac{\partial f( \bm{x}; \bm{W}^{(0)})}{\partial \bm{W}} = \Theta(1)\,.
\end{equation*}

Then we use triangle inequality, which concludes the proof.
\end{proof}

\begin{lemma}[Adapted from Lemma 8.2 in~\citet{allen2019convergence}]
\label{lemma:AZ_8_2}
Suppose that $\omega = \mathcal{O}(\frac{1}{L^{9/2}(\log m)^3})$, then w.p. at least $ 1-\exp(-\Omega(m\omega^{2/3}L))$ over random initialization, if $\left \| \bm{W}_{l} - \bm{W}_{l}' \right \|_2 \leq \omega, \forall l \in [L]$, we have $\left \| \bm{W}_l\bm{h}_{i,l-1} - \bm{W}_l'\bm{h}_{i,l-1}' \right \|_2 = \mathcal{O} (\omega L^{5/2}\sqrt{\log m}), \forall l \in [L]$.
\end{lemma}

For notational simplicity, we define the notation $\bm{b}_l$:

\begin{equation*}
\bm{b}_l = \begin{cases}
  & 1 \quad \text{ if } l=L+1, \\
  & \bm{D}_l (\bm{W}_{l+1})^{\top} \bm{b}_{l+1} \quad \text{otherwise}.
\end{cases}
\end{equation*}

\begin{lemma}[Adapted from Lemma 8.7 in~\citet{allen2019convergence}]
\label{lemma:AZ_8_7}
Suppose that $\omega = \mathcal{O}(\frac{1}{L^{6}(\log m)^{3/2}})$, then with probability at least $1-\exp(-\Omega(\omega^{2/3} mL \log m))$ over random initialization, if $\left \| \bm{W}_{l} - \bm{W}_{l}' \right \|_2 \leq \omega, \forall l \in [L]$, we have $\left \| \bm{b}_l - \bm{b}_l'  \right \|_2 = \mathcal{O}(\omega^{1/3}L^2\sqrt{m \log m}), \forall l \in [L]$.
\end{lemma}

\begin{proof}
    According to~\citet[Lemma 8.7]{allen2019convergence}, choose $s:=m\omega^{2/3}L$, which concludes the proof.
\end{proof}

\begin{lemma}
\label{lemma:lazy_to_Gradient_Perturbation}
Suppose that $\omega = \mathcal{O}(\frac{1}{L^{6}(\log m)^3})$, then with probability at least $1-\exp(-\Omega(\omega^{2/3} mL))$ over random initialization, if $\left \| \bm{W}_{l} - \bm{W}_{l}' \right \|_2 \leq \omega, \forall l \in [L]$, we have:
\begin{equation*}
    \left \| \bm{b}_l'(\bm{W}_{l-1}'\bm{h}_{i,l-2}')^{\top} - \bm{b}_l(\bm{W}_{l-1}\bm{h}_{i,l-2})^{\top} \right \|_{\mathrm{F}} = \mathcal{O}(\omega^{1/3}L^{5/2}\sqrt{m \log m})\,, \quad \forall l \in [L] \,.
\end{equation*}
\end{lemma}

\begin{proof}
We use~\cref{lemma:AZ_8_2,lemma:AZ_8_7} and the triangle inequality:
\begin{equation*}
\begin{split}
    & \left \| \bm{b}_l'(\bm{W}_{l-1}'\bm{h}_{i,l-2}')^{\top} - \bm{b}_l(\bm{W}_{l-1}\bm{h}_{i,l-2})^{\top} \right \|_{\mathrm{F}} \\
    \leq & \left \| \bm{b}_l'(\bm{W}_{l-1}'\bm{h}_{i,l-2}')^{\top} - \bm{b}_l(\bm{W}_{l-1}'\bm{h}_{i,l-2}')^{\top} \right \|_{\mathrm{F}}+\left \| \bm{b}_l(\bm{W}_{l-1}'\bm{h}_{i,l-2}')^{\top} - \bm{b}_l(\bm{W}_{l-1}\bm{h}_{i,l-2})^{\top} \right \|_{\mathrm{F}}\\
    \leq & \mathcal{O}(\omega^{1/3}L^{5/2}\sqrt{m \log m})\,.
\end{split}
\end{equation*}
\end{proof}

\begin{lemma}[Adapted from Lemma F.9 in~\citet{arora2019exact}]
\label{lemma:lazy_training}
Let $\omega \leq \text{poly}(\epsilon, L, \lambda_0, \frac{1}{\log(m)}, \frac{1}{\log(1/\delta)},\kappa,\frac{1}{n}).$ If $m \geq \text{poly}(1/\omega)$, then with probability at least $1-\delta$ over random initialization, we have:
\begin{equation*}
    \left \| \bm{W}_l^{(t)} - \bm{W}_l^{(0)} \right \|_{\mathrm{F}} \leq \omega\,, \quad \forall t \geq 0, \forall l \in [L]\,,
\end{equation*}
and
\begin{equation*}
    f_{\tt nn}^{(t)} (\bm{x}) - y \leq \exp(-\frac{1}{2}\kappa^2\lambda_0 t)(f_{\tt nn}^{(0)} (\bm{x})- y)\,.
\end{equation*}
\end{lemma}

\begin{lemma}[Kernel Perturbation During Training]
\label{lemma:Kernel_Perturbation_During_Training}
Suppose that, $\omega \leq poly(1/n, \lambda_0, 1/L, 1/\log(m), \epsilon, 1/\log(1/\delta))$, $m\geq poly(1/\omega)$ and $\kappa \leq 1$. Then with probability at least $1-\delta$ over random initialization, we have for all $t \leq 0$, $\forall (\bm{x}_1,\bm{x}_2)$:
\begin{equation*}
    \left | K_{\tt NTK}^{(t)}(\bm{x}_1,\bm{x}_2) - K_{\tt NTK}^{(0)}(\bm{x}_1,\bm{x}_2) \right | \leq \mathcal{O}(\omega^{1/3}L^{5/2}\sqrt{m \log m})\,.
\end{equation*}
\end{lemma}

\begin{proof}

By~\cref{lemma:lazy_training}, we know that for $t \to \infty$, $\left \| \bm{W}_l^{(t)} - \bm{W}_l^{(0)} \right \|_{\mathrm{F}} \leq \omega$, by~\cref{lemma:lazy_to_Gradient_Perturbation}, we know that on the gradient, there is only a small perturbation, then the perturbation of kernel value is small by~\cref{lemma:Gradient_Perturbation_to_Kernel_Perturbation}.
\end{proof}

\begin{lemma}[Kernel Perturbation $\to$ Output Perturbation. Adapted from Lemma F.1 in~\citet{arora2019exact}]
\label{lemma:Kernel_Perturbation_to_Output_Perturbation}
Fix $\epsilon_{\bm{H}}\leq \frac{1}{2}\lambda_0$. Suppose $\left | f( \bm{x}_i; \bm{W}^{(0)}) \right |\leq \epsilon_{0}, \forall i \in [n]$, $\left | f( \bm{x}_{\text{te}}; \bm{W}^{(0)}) \right |\leq \epsilon_{0}$ and $f_{\tt nn}^{(0)}(\bm{x}) - y = \mathcal{O}(\sqrt{n})$. Furthermore, if $\forall t\geq0$, $\left \| \bm{K}_{\tt NTK}^{(t)}(\bm{x}_{\text{te}},\bm{X}) - \bm{K}_{\tt NTK}^{(0)}(\bm{x}_{\text{te}},\bm{X})  \right \|_2 \leq \epsilon_{te}$ and $\left \| \bm{K}_{\tt NTK}^{(0)} - \bm{K}_{\tt NTK}^{(t)} \right \|_2 \leq \epsilon_{\bm{H}}$, then we have:
\begin{equation*}
    \left | f_{\tt nn}(\bm{x}_{te}) - f_{\tt NTK}(\bm{x}_{te}) \right | \leq \mathcal{O} \left(\epsilon_0 + \frac{\sqrt{n}}{\lambda_0}\epsilon_{te}+\frac{\sqrt{n}}{\lambda_0^2}\log \left(\frac{n}{\epsilon_{\bm{H}}\lambda_0 \kappa}\right)\epsilon_{\bm{H}} \right)\,.
\end{equation*}
\end{lemma}

Then we are ready to prove~\cref{prop:equivalence_ntk_network}.

\begin{proof}[Proof of~\cref{prop:equivalence_ntk_network}]
According to~\cref{lemma:Kernel_Perturbation_to_Output_Perturbation}, we have:
\begin{equation*}
    \left | f_{\tt nn}(\bm{x}_{te}) - f_{\tt NTK}(\bm{x}_{te}) \right | \leq \mathcal{O} \left(\epsilon + \frac{\sqrt{n}}{\lambda_0}\epsilon_{te}+\frac{\sqrt{n}}{\lambda_0^2}\log \left(\frac{n}{\epsilon_{\bm{H}}\lambda_0 \kappa}\right)\epsilon_{\bm{H}} \right)\,.
\end{equation*}

Note that the function $g(x) := x\log(\frac{n}{x\lambda_0\kappa})$ achieves its maximum $g(x)_{\max} = \frac{n}{e \lambda_0 \kappa}$.

Combine this and~\cref{lemma:Kernel_Perturbation_During_Training}, we have:

\begin{equation*}
    \left | f_{\tt nn}(\bm{x}_{te}) - f_{\tt NTK}(\bm{x}_{te}) \right | \leq \mathcal{O} \left(\epsilon + \frac{\sqrt{n}}{\lambda_0}\omega^{1/3}L^{5/2}\sqrt{m \log m}+\frac{n^{3/2}}{\lambda_0^3\kappa}\right)\,.
\end{equation*}

\end{proof}

\subsection{Proof of~\cref{prop:diff_kernel_ground_truth}}
\label{sec:proof_of_diff_kernel_ground_truth}

Before proving~\cref{prop:diff_kernel_ground_truth}, we first introduce the following lemmas:

\begin{lemma}[Adapted from Theorem 1 in~\citet{chen2021deep}]
\label{lemma:equivalence_between_ntk_laplace}
    Let $\mathcal{H}_{\tt Lap}$ and $\mathcal{H}_{\tt NTK}$ be the RKHS associated with the Laplace kernel and the neural tangent kernel of a $L$-layer fully connected ReLU network. Both kernels are restricted to the sphere $\mathbb{S}^{d-1}$. Then the two spaces include the same set of functions:
\begin{equation*}
    \mathcal{H}_{\tt Lap} = \mathcal{H}_{\tt NTK}\,.
\end{equation*}
\end{lemma}

Then we are ready to prove~\cref{prop:diff_kernel_ground_truth}.

\begin{proof}[Proof of~\cref{prop:diff_kernel_ground_truth}]
    According to~\cref{lemma:equivalence_between_ntk_laplace}, we have $\mathcal{H}_{\tt Lap} = \mathcal{H}_{\tt NTK}$, that means, the estimators of \cref{eq:ERM} under two RKHS corresponding to the NTK kernel and Laplace kernel are the same:
    \begin{equation}
        f_{\tt NTK}(\bm{x}_{te}) = f_{\tt Laplace}(\bm{x}_{te})\,,
    \label{eq:eq_ntk_laplace}
    \end{equation}
    where $f_{\tt NTK}$, $f_{\tt Laplace}$ are the estimators of \cref{eq:ERM} in $\mathcal{H}_{\tt NTK}$ and $\mathcal{H}_{\tt Laplace}$, respectively.

According to \cref{eq:laplace_is_dot}, when the Laplace kernel is restricted to the sphere $\mathbb{S}^{d-1}$, it is a dot product kernel. 

Combine~\cref{eq:eq_ntk_laplace} and~\citet[Theorem 2]{liu2021kernel}, we get the result. 
    
\end{proof}

\subsection{Proof of~\cref{thm:precise_form}}
\label{sec:proof_of_main_thm}

\begin{proof}

Using triangle inequality, we have: 
\begin{equation}
\begin{split}
    \mathbb{E} \left \| f_{\tt nn} -f_{\rho} \right \|_{L_{\rho_X}^2}^{2} \leq & \mathbb{E} \left \| f_{\tt nn} -f_{\tt ntk} \right \|_{L_{\rho_X}^2}^{2} + \mathbb{E} \left \| f_{\tt ntk} -f_{\rho} \right \|_{L_{\rho_X}^2}^{2}\\
    \leq & n^{-2\theta }\log^4 (\frac{2}{\delta})  + \frac{\sigma_{\epsilon}^2\beta}{d}\mathcal{N}_{\widetilde{\bm{X}}}^{\gamma} + \frac{\sigma_{\epsilon}^2 \log ^{2+4\varepsilon}d}{\gamma^2d^{4\theta-1}} \\
    + & \mathbb{E} \int_{X}\left | f_{\tt ntk}(\bm{x}) -f_{\rho}(\bm{x}) \right |^2d\rho_{X}(\bm{x})\\
    \leq & n^{-2\theta }\log^4 (\frac{2}{\delta})  + \frac{\sigma_{\epsilon}^2\beta}{d}\mathcal{N}_{\widetilde{\bm{X}}}^{\gamma} + \frac{\sigma_{\epsilon}^2 \log ^{2+4\varepsilon}d}{ \gamma^2d^{4\theta-1}} \\
    + & \mathcal{O} \bigg(\big(\epsilon + \frac{\sqrt{n}}{\lambda_0}\omega^{1/3}L^{5/2}\sqrt{m \log m}+\frac{n^{3/2}}{\lambda_0^3\kappa}\big)^2\bigg)\,,
\end{split}
\label{eq:diff_nn_rho}
\end{equation}

with probability at least $1-2\delta-\delta'-d^{-2}$, where the first inequality use~\cref{prop:diff_kernel_ground_truth} and second inequality use~\cref{prop:equivalence_ntk_network}.

Then According to~\cref{thm:min_eigen_NTK}, we have:

\begin{equation}
\lambda_0 \geq \begin{cases}
  & 2\mu_{1}^{2}\frac{n}{d}\bigg(\frac{3}{4}-\frac{c}{4}\sqrt{\frac{d}{n}}\bigg)^2, \quad \text{if} \quad n \geq d, \\
  & 2\mu_{1}^{2}\frac{n}{d}\bigg(\sqrt{\frac{d}{n}}-\frac{c+6}{4}\bigg)^2, \quad \text{if} \quad n < d\,.
\end{cases}
\label{eq:min_eigen}
\end{equation}

with probability at least $1-2e^{-n}$.

Combine~\cref{eq:diff_nn_rho} and \cref{eq:min_eigen}, we conclude the proof. 
\end{proof}

\end{document}